\documentclass[twoside,11pt]{article}

\usepackage{jmlr2e}

\usepackage{algorithm,algcompatible}
\usepackage{comment} 
\usepackage{lipsum} 
\usepackage{graphics} 
\usepackage{epsfig} 
\usepackage{times} 
\usepackage{amsmath} 
\usepackage{amssymb}  
\usepackage{color}
\usepackage{xcolor}
\usepackage{caption}
\usepackage{subcaption}
\usepackage{mathabx}
\usepackage{multirow}
\usepackage{cleveref}

\allowdisplaybreaks[3]
\newcommand{\norm}[1]{\left\lVert#1\right\rVert}
\newtheorem{assumption}{Assumption}

\makeatletter
\newlength\tdima
\newcommand\tabfill[1]{%
      \setlength\tdima{\linewidth}%
      \addtolength\tdima{\@totalleftmargin}%
      \addtolength\tdima{-\dimen\@curtab}%
      \parbox[t]{\tdima}{#1\ifhmode\strut\fi}}
\makeatother

\firstpageno{1}

\begin{document}

\title{Momentum Q-learning with Finite-Sample Convergence Guarantee}

\author{\name Bowen Weng\thanks{equal contribution} \email weng.172@osu.edu \\
       \addr Department of Electrical and Computer Engineering\\
       The Ohio State University\\
       Columbus, OH 43210, USA
       \AND
       \name Huaqing Xiong$^*$ \email xiong.309@osu.edu \\
       \addr Department of Electrical and Computer Engineering\\
       The Ohio State University\\
       Columbus, OH 43210, USA
       \AND
       \name Lin Zhao$^*$ \email elezhli@nus.edu.sg \\
       \addr Department of Electrical and Computer Engineering \\
       National University of Singapore \\
       Singapore 117583, Republic of Singapore
       \AND
       \name Yingbin Liang \email liang.889@osu.edu \\
       \addr Department of Electrical and Computer Engineering\\
       The Ohio State University\\
       Columbus, OH 43210, USA
       \AND
       \name Wei Zhang \email zhangw3@sustech.edu.cn \\
       \addr Department of Mechanical and Energy Engineering\\
       Southern University of Science and Technology (SUSTech)\\
       Shenzhen, 518055, China}


\maketitle

\begin{abstract}
Existing studies indicate that momentum ideas in conventional optimization can be used to improve the performance of Q-learning algorithms. However, the finite-sample analysis for momentum-based Q-learning algorithms is only available for the tabular case without function approximations. This paper analyzes a class of momentum-based Q-learning algorithms with finite-sample guarantee. Specifically, we propose the MomentumQ algorithm, which integrates the Nesterov's and Polyak's momentum schemes, and generalizes the existing momentum-based Q-learning algorithms. 
For the infinite state-action space case, we establish the convergence guarantee for MomentumQ with linear function approximations and Markovian sampling. In particular, we characterize the finite-sample convergence rate which is provably faster than the vanilla Q-learning. This is the first finite-sample analysis for momentum-based Q-learning algorithms with function approximations. For the tabular case under synchronous sampling, we also obtain 
a finite-sample convergence rate that is slightly better than the SpeedyQ \citep{azar2011speedy} when choosing a special family of step sizes.
Finally, we demonstrate through various experiments that the proposed MomentumQ outperforms other momentum-based Q-learning algorithms. 

\end{abstract}


\section{Introduction}

Reinforcement learning (RL) aims to design strategies for an agent to find a desirable policy through interacting with an environment in order to maximize an accumulative reward for a task. RL has received drastically growing attention in recent years and accomplished tremendous success in various application domains such as playing video games~\citep{mnih2013playing}, bipedal walking robot~\citep{castillo2018reinforcement}, board game~\citep{Silver2017}, to name a few. This paper focuses on Q-learning, which is a widely used model-free RL algorithm for finding the action-value function (known as the Q-function) of the optimal policy. 

Q-learning was first proposed in~\cite{watkins1992q} and has been studied extensively since then. For scenarios with a finite state-action space, the Q-function can be conveniently represented as a tabular function. The convergence of Q-learning in the tabular case was proved in~\cite{jaakkola1994convergence}. 
In the case with a continuous state-action space, one typically approximates the Q-function with a parameterized function class of a relatively small parameter dimension. Among the rich approximation classes, linear function approximation \citep{bertsekas1996neuro,sutton2018reinforcement} and neural network function approximation \citep{mnih2013playing} are often adopted in the literature. We will review these studies in more details in Section \ref{subsec:relatedWork}.

The central idea of Q-learning algorithms is to solve an optimal Bellman equation~\citep{bertsekas1996neuro} iteratively as a fixed point problem. Since the Bellman operator is expressed as the expected value over the underlying Markov decision processes (MDP) which is unknown, Q-learning (as a model-free algorithm) approximates it via its sampled version, and such an update can be viewed analogously to the first-order (stochastic) gradient descent algorithm~\citep{baird1995residual}. 
This connection thus motivated several studies on accelerating Q-learning by incorporating various momentum schemes, such as Heavy-ball (HB)~\citep{polyak1964some} and Nesterov's accelerated gradient (NAG)~\citep{nesterov2013introductory} which were shown to accelerate gradient descent in conventional optimization algorithms. For example, speedy Q-learning (SpeedyQ) proposed in~\citep{azar2011speedy} can be viewed as incorporating the NAG to Q-learning with particularly designed learning rate. \cite{devraj2019nesa} applied both HB and NAG to Q-learning with a matrix learning rate. 
\cite{vieillard2019momentum} incorporated the momentum idea to value iteration by viewing the greedy policy as an analog of gradient ascent.
However, theoretical justification of these momentum-based Q-learning are very limited. Only \cite{azar2011speedy} provided a finite-sample analysis in the {\em tabular} case under particularly chosen learning rate, whereas \cite{devraj2019nesa} provided only the asymptotic property without provable finite-sample convergence. To the best of our knowledge, the {\em finite-sample} convergence rate has not been established for momentum-based Q-learning algorithms with function approximation yet. The focus of the study here is to address the above important question.

\subsection{Main Contributions}

This paper investigates a general momentum-based Q-learning scheme (referred to as MomentumQ hereafter), which involves both NAG-type and HB-type of history information for accelerating Q-learning. The main contribution of this paper is three-fold.




First, we establish the finite-sample convergence rate for MomentumQ with linear function approximation, and we show that this algorithm provably accelerates vanilla Q-learning. To the best of our knowledge, this is the first finite-sample convergence guarantee for momentum-based Q-learning with linear function approximation. 

Second, the only existing finite-sample baseline bound for momentum-based Q-learning is given by SpeedyQ~\citep{azar2011speedy} for the tabular case. Hence, to be able to compare with such a baseline, we also provide a finite-sample analysis of MomentumQ in the tabular case. We show that it achieves a better (but order-wisely the same) convergence rate than SpeedyQ. Technically, due to the additional momentum terms in MomemtumQ, its analysis is more challenging than SpeedyQ and requires substantial new technical developments. 



Finally, our numerical results show that the proposed MomentumQ outperforms the vanilla Q-learning as well as the other existing momentum-based Q-learning algorithms for both tabular and function approximation cases.

\subsection{Related work}\label{subsec:relatedWork}
We review the most relevant studies on Q-learning here with a focus on the theoretical convergence analysis.

\textbf{Q-learning with function approximation:} When the state-action space is considerably large or even continuous, it is practical to properly discretize the space \citep{shah2018q}, or parameterize the Q-function with a certain function class. 
For function approximation with neural networks, \cite{yang2019theoretical} provided statistical results for a Deep-Q-Network (DQN)-type algorithm. \cite{lee2020periodicQ} further analyzed a similar variant with periodic target function update and established an improved sample complexity bound. 
For linear MDP, \cite{melo2007q,yang2019sample} proposed provably sample-efficient Q-learning algorithms with linear function approximation. For more general MDPs with linear function approximation of the Q-function, finite-sample convergence analysis was established in~\cite{zou2019finite,Chen2019finiteQ} under Markovian sampling, in \cite{du2019provably} on exploration samples and in \cite{ijcai2020-422} by incorporating Adam-type updates. Recently, \cite{cai2019neural,xu2019deepQ} established the convergence rate of Q-learning with neural network approximation in the overparameterized regime under i.i.d.\ and non-i.i.d\ sampling, respectively. 


\textbf{Tabular Q-learning:}
Q-learning was first proposed in \cite{watkins1992q} under finite state-action space. 
Regarding the theoretical studies, research of tabular Q-learning has focused on the asymptotic convergence which was usually studied via its connection to the corresponding stochastic approximation algorithm (see, for example,~\cite{tsitsiklis1994asynchronous,jaakkola1994convergence,borkar2000ode, melo2001convergence}). More recently, \cite{Lee2019Switch} provided asymptotic results for asynchronous Q-learning by formulating it as a switching affine system. Another research line has focused on the finite-sample (i.e., non-asymptotic) analysis. Finite-sample performance for Q-learning was first established in \cite{szepesvari1998asymptotic}. Considering both synchronous and asynchronous Q-learning,  \cite{even2003learning} investigated the convergence rates under different choices of the learning rates. Sharper bounds on the finite-sample convergence rate have been established in more recent work \citep{wainwright2019stochastic,qu2020finite,li2020sample}.

\textbf{Momentum-based Q-learning:} For tabular Q-learning, several studies incorporated the momentum idea in conventional optimization to accelerate the convergence. \cite{azar2011speedy} proposed the SpeedyQ algorithm and characterized the finite-sample performance. 
\cite{devraj2019nesa} extended HB with a matrix learning rate on the momentum, which is similar to a special formulation of NAG. The asymptotic performance was analyzed under simplified assumptions. 
~\cite{vieillard2019momentum} proposed a momentum-based value iteration and generalized the scheme to DQN. While some theoretical properties of the algorithms were explored in the tabular case, the convergence of the algorithm was not established. Among these studies, only \cite{azar2011speedy} characterized the finite-sample rate for SpeedyQ in the tabular case, and such finite-sample analysis for momentum-based Q-learning algorithms has not been provided for the function approximation case, which is the focus of this paper.

\textbf{Other variants of Q-learning:} Other than the above momentum-based Q-learning algorithms, which mainly exploit the acceleration ideas in conventional optimization, Q-learning also inspires a number of other variants, including residual Q-learning~\citep{baird1995residual}, phased Q-learning~\citep{kearns1999finite}, Zap Q-learning~\citep{devraj2017zap}, and periodic Q-learning~\citep{lee2020periodicQ}, to name a few. These algorithms are proposed to speed up convergence rates or improve the performance by mitigating various issues in the implementation of Q-learning. In this paper, we mainly focus on the momentum-based Q-learning algorithm motivated by the optimization idea.





\subsection{Organization}

The rest of the paper is organized as follows. Section~\ref{sec: tabularQ} reviews the background of Q-learning. In section~\ref{sec:momentumQ}, we propose the MomentumQ algorithm, described in both tabular case and under linear function approximation.
Section~\ref{sec:ctnQ} establishes the finite-sample convergence guarantee for the proposed MomentumQ under linear function approximation, followed by the finite-sample analysis for the tabular case. 
Section~\ref{sec:exp} numerically evaluates the proposed algorithm and compares it with several other algorithms via experiments of a series of FrozenLake grid world games.

\section{Preliminaries}\label{sec: tabularQ}

In this section, we provide the background of the Markov decision process, followed by the preliminaries of tabular Q-learning and then Q-learning with linear function approximation.

\subsection{Markov Decision Process}\label{subsec:MDP}
We consider the standard reinforcement learning setting, where a learning agent interacts with a (possibly stochastic) environment modeled as a discrete-time discounted Markov decision process (MDP). Such an MDP is characterized by a quintuple $(\mathcal{X},\mathcal{U},P,R, \gamma)$, where $\mathcal{X}$ is the state space, $\mathcal{U}$ is the action space, $P:\mathcal{X}\times \mathcal{U} \times \mathcal{X}\mapsto [0,1]$ is the probability transition kernel, namely, $P(\cdot|x, u)$ denotes the probability that the system takes the next state given the current state $x$ and action $u$. In addition, $R: \mathcal{X}\times \mathcal{U}\mapsto[0,R_{\max}]$ denotes the reward function (or negative of the cost function) mapping the state-action pairs to a bounded subset of $\mathbb{R}$, and $\gamma\in (0,1)$ is the discount factor. A policy $\pi:\mathcal{X}\mapsto\mathcal{U}$ represents a strategy to take actions, i.e., it captures the probability of taking each action at any given state. By following a policy $\pi$, we perform an action $u_k$ with probability $\pi(u_k|x_k)$ at time $k$, observe a reward $r_k=R(x_k,u_k)$, and evolve to the next state $x_{k+1}$ with the probability $P(x_{k+1}|x_k,u_k)$. Under the policy $\pi$, the return is the sum of the observed rewards over the entire time horizon. 
We define the value function as the expected return of following policy $\pi$ and starting from state $x$, given by $V^\pi(x)={\mathbb{E}_P\sum_{k=0}^\infty\gamma^k{r_k}}$, where $\mathbb{E}_P$ denotes the expectation with respect to the transition probability $P$. The $Q$-function is defined as the state-action value function  $Q^{\pi}(x,u)=R(x,u)+\gamma \sum_{y\in\mathcal{X}}P(y|x, u)V^{\pi}(y)$, which is the return of performing action $u$ at state $s$ at the first step and following policy $\pi$ thereafter. 

\subsection{Tabular Q-learning}
Q-learning seeks to maximize the expected discounted return over policy $\pi$ as formulated below.
\begin{align}
    & \underset{\pi}{\text{maximize}}
    & & V^{\pi}(x_0) = {\mathbb{E}_P}\left[\sum_{k=0}^{\infty} \gamma^k R(x_k, \pi(x_k))\right], \nonumber\\
    & \text{subject to}
    & & x_{k+1} \sim P(\cdot|x_k, \pi(x_k)),\label{eq:systemEquation}
\end{align}
 We let $\pi^{\star}$ denote the optimal stationary policy $\pi^{\star}:\mathcal{X}\mapsto \mathcal{U}$ of MDP which is the solution of the above optimization problem. 
 
 
Define the Bellman operator $\mathcal{T}$ pointwisely as
\begin{equation}
    \label{eq:BellmanOperator}
    \mathcal{T}Q(x,u) = R(x,u)+\gamma\mathbb{E}_P \underset{u^\prime \in U(x^\prime)}{\max}Q(x^\prime,u^\prime),
\end{equation}
where $x^\prime \sim P(\cdot | x,u)$ and $U(x)$ denotes the admissible set of actions at state $x$. It can be shown that the Bellman operator $\mathcal{T}$ is $\gamma$-contractive in the supremum norm $\norm{Q}:=\sup_{x,u}|Q(x,u)|$, i.e., it satisfies
\begin{equation}
    \label{eq:Contraction}
    \norm{\mathcal{T}Q(x,u) - \mathcal{T}Q'(x,u)} \leq \gamma\norm{Q(x,u) - Q'(x,u)}.
\end{equation}
Thus, $\mathcal{T}$ has a unique fixed point $Q^{\star}$, which satisfies 
the \textit{optimal Bellman equation}~\citep{bertsekas1996neuro} given by 
\begin{equation}
    \label{eq:OptimalBellmanEq}
    Q^{\star}(x,u) = R(x,u)+\gamma\mathbb{E}_P \underset{u^\prime \in U(x^\prime)}{\max}Q^{\star}(x^\prime,u^\prime),
\end{equation}
and the associated policy $\pi^{\star}$ is the optimal solution of \eqref{eq:systemEquation}.
The above property suggests that starting with an arbitrary initial Q-function, we can apply the Bellman operator $\mathcal{T}$ iteratively to learn $Q^{\star}$. 

Let $V^{\star}(x):=V^{\pi^{\star}}(x)$ be the optimal value function corresponding to the optimal policy $\pi^{\star}$. It relates to $Q^{\star}$ as follows
\begin{equation}
    \label{eq:bellQ}
    V^{\star}(x)=\underset{u\in U(x)}{\max}Q^{\star}(x,u), \forall x\in\mathcal{X}.
\end{equation}
Hence, the optimal policy can be obtained from the optimal Q-function as:
\begin{equation}
    \label{eq:optPol}
    \pi^{\star}(x) = \underset{u \in U(x)}{\text{argmax}}\ Q^{\star}(x,u),\forall x\in\mathcal{X}.
\end{equation}
Note that the knowledge of the transition probability $P$ is not needed in~\eqref{eq:optPol}, which is one advantage of Q-learning.

In practice, exact evaluation of the Bellman operator~\eqref{eq:BellmanOperator} is usually not feasible due to the lack of the knowledge of the system dynamics (i.e. the transition probability kernel). Instead, the \textit{empirical Bellman operator} is used as an estimator based on samples~\citep{jaakkola1994convergence}. 
Specifically, for the $k$th round of iteration at the state-action pair $(x,u)$, we sample the next state $y_k\sim P(\cdot|x,u)$, and then evaluate the empirical Bellman operator $\hat{\mathcal{T}}_{k}$ as
\begin{equation}\label{eq:empTQ}
    \hat{\mathcal{T}}_{k}Q_k(x,u) = R(x,u)+\gamma \underset{u'\in U(y_k)}{\max}Q_k(y_k,u'),
\end{equation}
where the subscript $k$ in $\hat{\mathcal{T}}_k$ is to track the time index of samples $y_k$ that are used. Then the iteration of tabular Q-learning is implemented as
\begin{equation}\label{eq:tabular Q-learning}
    Q_{k+1} =  Q_k  - \alpha_k(Q_k - \hat{\mathcal{T}}_{k}Q_k),
\end{equation}
where $\alpha_k$ is the stepsize and we omit the dependence on $(x,u)$ for simplicity when there is no confusion.

\subsection{Q-learning with Linear Function Approximation}

For relatively large or even infinite state-action space $\mathcal{X}\times\mathcal{U}$, it is impractical to express the Q-function in an explicit tabular form with respect to each state-action pair. In such a case, the update rule of \eqref{eq:tabular Q-learning} is no longer directly applicable.

To handle such cases, a parametric function $\hat{Q}(x,u ; \theta)$ is adopted as an approximation of the Q-function, where the parameter vector $\theta$ is of small dimension. 
Our focus here is the linear function class, which is often considered in the literature for establishing the finite-sample analysis~\citep{zou2019finite,Chen2019finiteQ,du2019provably}. Then the Q-function $\hat{Q}(x, u;\theta)$ can be written as
\begin{equation}
    \label{eq:linearApprox}
    \hat{Q}(x, u;\theta) = \Phi(x, u)^T\theta, 
\end{equation}
where $\theta\in\mathbb{R}^d$, and $\Phi:\mathcal{S}\times \mathcal{A}\rightarrow \mathbb{R}^d$ is a vector function of size $d$, and the elements of $\Phi$ represent the nonlinear kernel (feature) functions. Correspondingly, the updating rule of Q-learning with linear function approximation is given by
\begin{equation}\label{eq:linearPQL}
    \theta_{k+1}\! =\! \theta_k - \alpha_k \Phi(x_k,u_k)\left[ \Phi(x_k,u_k)^T\theta_k - R(x_k,u_k) - \gamma \underset{u'\in U(x_{k+1})}{\max} \Phi(x_{k+1},u')^T\theta_k  \right],
\end{equation}
where $\alpha_k$ is the stepsize.

\section{MomentumQ Algorithm}\label{sec:momentumQ}

In this section, we introduce the MomentumQ algorithm that we study. 

\subsection{Tabular MomentumQ}

Overall, MomentumQ integrates the Nesterov's momentum~\citep{nesterov2013introductory} and Polyak's Momentum ~\citep{polyak1964some} together, with the learning rates flexibly interpolating between the two to optimize the momentum performance. Specifically, MomentumQ takes the form given by
\begin{equation}
    \begin{aligned}
    \label{eq: acc tabular Q-learning}
    & S_k = (1-a_k)Q_{k-1} + a_k\hat{\mathcal{T}}_{k}Q_{k-1},\\
    & P_{k} = (1-a_k)Q_k + a_k\hat{\mathcal{T}}_{k}Q_k, \\
    & Q_{k+1} = P_{k} \ \ + \underbrace{b_k(P_{k} - S_{k})}_\text{Nesterov's momentum} +\ \ \underbrace{c_k(Q_{k}-Q_{k-1})}_\text{Polyak's momentum}.
    \end{aligned}
\end{equation}
where $a_k,b_k,c_k$ determine the learning rates. 
Algorithm~\ref{alg:AQL} implements MomentumQ with a particular family of learning rates under synchronous sampling~\citep{even2003learning}. One special feature of the algorithm is the additional freedom introduced by the hyperparameter $m$. We will see later in the simulation that the proposed algorithm accelerates the convergence for arbitrarily chosen $m$ that satisfies $m\geq 1/\gamma$.
\begin{algorithm}
\caption{\label{alg:AQL}Synchronous Tabular MomentumQ}
\begin{tabbing}
\tabfill{\noindent\textbf{Input:} Initial action-value function $Q_0$ and $Q_{-1}=Q_0$, discount factor $\gamma$, hyperparameter $m\geq \frac{1}{\gamma}$, and maximum 
iteration number $T$} \\
\textbf{for} \= $k=0,1,2,\cdots, T-1$ \textbf{do} \\
\> $a_k = \frac{1}{k+1},\ b_{k}=k-m-1,\ c_{k}=\frac{-k^2+(m+1)k+1}{k+1}$; \\
\>\textbf{for} \= each $(x,u)\in \mathcal{X}\times U(x)$ \textbf{do}\\
\> \> Generate the next state sample $y_k\sim P(\cdot|x,u);$\\
\> \> $\hat{\mathcal{T}}_{k}Q_{k-1}(x,u)= R(x,u)+\gamma\max_{u\in U(y_k)}Q_{k-1}(y_k,u);$ \\
\> \> $\hat{\mathcal{T}}_{k}Q_{k}(x,u)= R(x,u)+\gamma\max_{u\in U(y_k)}Q_k(y_k,u);$ \\
\> \> $S_{k}(x,u)= (1-a_{k})Q_{k-1}(x,u)+a_{k}\hat{\mathcal{T}}_{k}Q_{k-1}(x,u);$ \\
\> \> $P_{k}(x,u)= (1-a_{k})Q_{k}(x,u)+a_{k}\hat{\mathcal{T}}_{k}Q_{k}(x,u);$ \\
\> \> $Q_{k+1}(x,u)= $ \= $P_{k}(x,u)+b_k\left(P_{k}(x,u)-S_{k}(x,u) \right)$ $+c_k(Q_k(x,u)-Q_{k-1}(x,u));$ \\
\> \textbf{end for} \\
\textbf{end for}\\
\textbf{Output:} $Q_T$
\end{tabbing}
\end{algorithm}

Note that the proposed MomentumQ algorithm in \eqref{eq: acc tabular Q-learning} contains not only the momentum term $\hat{\mathcal{T}}_{k}Q_{k-1}$ in the update, but also the historical information $Q_{k-1}$ explicitly. This additional historical information can smooth out large overshoots during the iteration and subsequently accelerate the convergence. This can be observed clearly in the experiment when compared to SpeedyQ, which is given by 
\begin{equation}\label{eq:sql}
    Q_{k+1} = Q_k + a_k(\hat{\mathcal{T}}_{k}Q_k - Q_k) + (1-a_k)(\hat{\mathcal{T}}_{k}Q_k - \hat{\mathcal{T}}_{k}Q_{k-1}).
\end{equation}
We see from~\eqref{eq:sql} that SpeedyQ contains only the momentum term $\hat{\mathcal{T}}_{k}Q_{k-1}$ in the update. In contrast, MomentumQ additionally incorporates the historical information $Q_{k-1}$ explicitly. Indeed, the simulation in \Cref{sec:exp} shows that MomentumQ effectively smoothes out the large overshoots that are present in SpeedyQ and converges faster. 
The finite-sample anlysis of MomentumQ is more challenging than SpeedyQ due to this difference, since the additional $Q_{k-1}$ term increases the order of the recursion. We will discuss in more details later.
Furthermore,~\eqref{eq:sql} simply involves $\hat{\mathcal{T}}_{k}Q_k - \hat{\mathcal{T}}_{k}Q_{k-1}$ as the only momentum term, while our algorithm designs this part more systematically. We directly use two consecutive outputs of the empirical Bellman operators to update the Q-function and obtain $S_k$ and $P_k$. Intuitively, since $S_k$ and $P_k$ are derived by the update of the vanilla Q-learning, selecting $S_k-P_k$ as the additional momentum term may contribute to a better estimation of the optimal Q-function while preserving the acceleration. This intuition is also verified in our numerical results, which will be shown in Section \ref{sec:exp}. 

\subsection{MomentumQ with Linear Function Approximation}\label{subsec:PAQL}

For the case where the state-action space is considerably large, suppose the linear function approximation is used for estimating the Q-function to overcome the curse of dimensionality. 

\begin{algorithm}
\caption{\label{alg:linearPAQL}MomentumQ with linear function approximation}
\begin{tabbing}
\noindent\textbf{Input:} Initial parameters $\theta_0$ and $\theta_{-1}=\theta_0$; discount factor $\gamma$; iteration number $T$.\\
\textbf{for} \= $k=0,1,2,\cdots, T-1$ \textbf{do} \\
\> Assign $a_k, b_{k}, c_k$; \\
\> Sample $u_k\sim\pi, x_{k+1}\sim P(\cdot|x_k,u_k)$;\\
\> Compute $g_k=\left(\Phi(x_k,u_k)^T\theta_k - R(x_k,u_k) - \gamma \underset{u'\in U(x_{k+1})}{\max} \Phi(x_{k+1},u')^T\theta_k\right)\Phi(x_k,u_k)$;\\
\> Update $\theta_{k+1}=\theta_k + (b_k + c_k)(\theta_k - \theta_{k-1}) -a_k(1+b_k)g_k + a_k b_k g_{k-1}$;\\
\textbf{end for}\\
\textbf{Output:} $\theta_{\text{out}}$.
\end{tabbing}
\end{algorithm}

Consider the case where the Q-function is approximated by a linear parameterized function. We propose MomentumQ for this case as 
\begin{equation}\label{eq:accLinear}
    \begin{aligned}
        \theta_{k+1} = \theta_k + (b_k + c_k)(\theta_k - \theta_{k-1}) -a_k(1+b_k)g_k + a_k b_k g_{k-1},
    \end{aligned}
\end{equation}
where 
\begin{align}
g_k :=& g(\theta_k;x_k,u_k,x_{k+1})\nonumber\\
=& \left(\Phi(x_k,u_k)^T\theta_k - R(x_k,u_k) - \gamma \underset{u'\in U(x_{k+1})}{\max} \Phi(x_{k+1},u')^T\theta_k\right)\Phi(x_k,u_k).\label{eq:gkQ}
\end{align}


We focus on the more practical Markovian sampling model, in which the data tuples are sequentially drawn from a single trajectory under an unknown stationary distribution. More implementation details can be referred to Algorithm \ref{alg:linearPAQL}.

\section{Finite-sample Analysis under Markovian Sampling}\label{sec:ctnQ}


In this section, we present our main results on the finite-sample convergence rate guarantee for MomentumQ. We focus on linear function class and provide the first finite-sample analysis for momentum-based Q-learning with function approximation.
We also present our study of tabular MomentumQ in order to make a comparison with the only existing theory baseline for momentum-based Q-learning, which was established for tabular SpeedyQ. 

\subsection{MomentumQ with Linear Function Approximation}


In this section, we characterize the finite-sample convergence guarantee for the proposed MomentumQ algorithm with linear function approximation under Markovian sampling. To proceed the convergence analysis, we first define 
\begin{align}
\bar g(\theta) :=& \underset{\mu}{\mathbb{E}}[g(\theta;x,u,x')]\nonumber\\
=& \underset{\mu}{\mathbb{E}}\left[ (\Phi(x,u)^T\theta - R(x,u) - \gamma \underset{u'\in U(x')}{\max} \Phi(x',u')^T\theta)\right]\Phi(x,u),\label{eq:gbarQ}
\end{align}
where the expectation is taken over the stationary distribution of the sampling tuple $(x,u,x')$. 

We take the following standard assumptions in our analysis.
\begin{assumption}
    \label{asp:boundedPhi}
    The columns of $\Phi$ are linearly independent and $\|\Phi\|_2 \leq 1$.
\end{assumption}
\begin{assumption}
    \label{asp:qlearning}
    The term $\bar g(\cdot)$ has a unique root denoted as $\theta^\star$, i.e., $\bar g(\theta^\star)=0$. There exists a constant $\delta>0$, such that for any $\theta\in\mathbb{R}^d$ we have
    \begin{equation}\label{eq:qasp}
        (\theta - \theta^\star)^T\bar g(\theta)\geq \delta\norm{\theta - \theta^\star}_2^2.
    \end{equation}
\end{assumption}
\begin{assumption}
    \label{asp:boundedDomain}
   The domain of the approximation parameters $\theta$ is contained in a ball $\mathcal{B}$ that includes  $\theta^\star$ and is centered around $\theta = 0$ with a bounded diameter. That is, there exists $D_{\max}$, such that $\norm{\theta - \theta'}_2 \leq D_{\max}, \forall \theta, \theta'\in \mathcal{B}$, and $\theta^\star\in\mathcal{B}$.
\end{assumption}
\begin{assumption}\label{asp:markov}
    There exist constants $\sigma > 0$ and $\rho \in (0,1)$ such that
    $$ \underset{x\in\mathcal{X}}{\sup}\ d_{TV}(\mathbb{P}(x_k\in \cdot|x_0 = x) ,\mu)\leq \sigma\rho^k \quad \forall k,
    $$
    where $d_{TV}(\mu,\nu)$ denotes the total-variation distance between the probability measures $\mu$ and $\nu$.
\end{assumption}
Assumptions \ref{asp:boundedPhi} and \ref{asp:qlearning} are standard in the literature on theoretical analysis of Q-learning algorithms with linear function approximation \citep{bhandari2018finite,Chen2019finiteQ,zou2019finite}. The boundedness condition in Assumption \ref{asp:boundedPhi} can be justified by normalization and hence does not lose generalization. Assumption \ref{asp:markov} can easily hold for irreducible and aperiodic Markov chains, and is widely adopted in the literature on theoretical analysis of RL algorithms under Markovian sampling \citep{bhandari2018finite,Chen2019finiteQ,zou2019finite,xu2019deepQ,xiong2020amsgradRL}. For Assumption \ref{asp:markov}, we further define the quantity of the mixing time $\tau^{mix}(\cdot)$ as follows, which denotes the duration of the time for the Markov chain to approach sufficiently close to its steady-state 
\begin{equation}\label{eq:mixtime}
    \tau^*:=\tau^{mix}(\kappa) := \min\left\{ k=1,2,\dots |\sigma\rho^{k}\leq\kappa \right\}.
\end{equation}

To understand the challenges of analyzing Markovian sampling in MomemtumQ, we first illustrate how a non-zero bias is introduced if the Markovian sampling is considered. For simplicity, we denote $O_k:=(x_k,u_k,x_{k+1})$ as the data at time step $k$ sampled from a Markov chain. Recall $g_k(\theta;O_k)$ in \eqref{eq:gkQ}, and $\bar g(\theta) = \mathbb{E}[g(\theta;O_k)]$ in \eqref{eq:gbarQ} where the expectation is taken over the marginal distribution of $O_k$ since $\theta$ is fixed. However, if $\theta$ is random and dependent on $O_k$, the equality no longer holds. In particular, since $\theta_k$ is dependent on the historical tuples $\{O_1,O_2,\dots,O_k\}$, we have 
$$ \bar g(\theta_k)\neq \mathbb{E}[g(\theta_k; O_k)|\theta_k].
$$
Thus, we have a non-zero bias due to Markovian sampling to approximate the expectation of $g_k^T(\theta_k-\theta^\star)$. Namely,
$$ \mathbb{E}[g_k^T(\theta_k-\theta^\star)] = \mathbb{E}[\bar g(\theta_k)^T(\theta_k-\theta^\star)] + \mathbb{E}[(g_k - \bar g(\theta_k))^T(\theta_k-\theta^\star)],
$$
where the second term on the right hand side captures the bias, which is the key challenge of the analysis under this setting. The following lemma develops an important upper bound on the bias term, which is a key step in the convergence analysis.
\begin{lemma}\label{lem:bias}
Suppose that Assumptions \ref{asp:boundedPhi}-\ref{asp:markov} hold and fix $\kappa>0$ in \eqref{eq:mixtime}. Let MomentumQ update as \eqref{eq:accLinear} by choosing non-increasing $a_k,b_k,c_k$ and denote $\beta_k = b_k+c_k$ with $\beta_k\in(0,1)$. Then we have
\begin{equation*}
    \mathbb{E}[(g_k - \bar g(\theta_k))^T(\theta_k-\theta^\star)]\leq
    \begin{cases}
    \eta_1\sum_{i=1}^{k-1}\beta_i + \eta_2\sum_{i=1}^{k-1}a_i, \quad & k\leq \tau^*;\\
    4D_{\max}G_{\max}\kappa + \eta_1\tau^*\beta_{k-\tau^*} + \eta_2\tau^*a_{k-\tau^*}, \quad & k>\tau^*,
    \end{cases}
\end{equation*}
where $\eta_1 = 2D_{\max}((1+\gamma)D_{\max}+G_{\max}),\eta_2 = 6G_{\max}((1+\gamma)D_{\max}+G_{\max})$ with $G_{\max} = 2D_{\max} + R_{\max}$.
\end{lemma}

With the bias term bounded, we are ready to provide the convergence result for MomentumQ with linear function approximation under Markovian sampling.
\begin{theorem}\label{thm:acclinearNoniid}
    (MomentumQ with constant learning rate) Suppose that Assumptions \ref{asp:boundedPhi}-\ref{asp:markov} hold and fix $\kappa>0$ in \eqref{eq:mixtime}. Let $a_k=\alpha, b_k+c_k=\beta\lambda^k$ where $\beta,\lambda\in (0,1)$ and $\alpha\in(0,\frac{1-\lambda}{2\delta})$. After running $T$ steps of Algorithm \ref{alg:linearPAQL}  under Markovian sampling, we take the output $\theta_{\text{out}} = \theta_T$ and have
    \begin{align}
    \mathbb{E}\norm{\theta_{\text{out}}-\theta^\star}_2^2
    &\leq \prod_{i=0}^{T-1} (1 - 2\alpha\delta(1+b_i))\norm{\theta_{0}-\theta^\star}_2^2 \nonumber\\
    &\quad + \beta\left( \frac{2\eta_1\tau^*}{\delta} + \frac{5 D_{\max}^2 + 2\alpha D_{\max}G_{\max} + 4\alpha \eta_1\tau^*\lambda }{1-2 \alpha\delta - \lambda} \right)(1-2 \alpha\delta)^{T-1-\tau^*} \nonumber \\
    &\quad + \frac{15 G_{\max}^2\alpha}{2\delta} + \frac{2\eta_2\tau^*\alpha}{\delta} +   \frac{8D_{\max}G_{\max}\kappa}{\delta}, \label{eq:thmConstantPAQLNoniid}
    \end{align}
    where $\eta_1,\eta_2$ are defined in Lemma \ref{lem:bias}.
\end{theorem}

Theorem \ref{thm:acclinearNoniid} indicates that the convergence behavior is determined by five terms. The first two terms capture the convergence rate as $T$ changes, indicating that with a constant learning rate, MomentumQ enjoys an exponential convergence rate to a neighborhood of the global optimum. Since $\prod_{i=0}^{T-1} (1 - 2\alpha\delta(1+b_i))<(1 - 2\alpha\delta)^{T}$, the dominant term of the convergence rate is the second term. The last three terms capture the convergence error. Since one usually chooses $\kappa = \alpha_k = \alpha$, the convergence error can be made as small as possible by choosing a sufficiently small learning rate. 

As a comparison, the convergence of the vanilla Q-learning under similar assumptions and Markovian sampling is obtained in \cite{Chen2019finiteQ} as $\mathbb{E}\norm{\theta_{\text{out}}-\theta^\star}_2^2\leq (1-2\delta \alpha)^T\norm{\theta_{0}-\theta^\star}_2^2 + \alpha C_1 + \kappa C_2$ for some constants $C_1,C_2$. Clearly, the dominant order in \eqref{eq:thmConstantPAQLNoniid} can have a smaller coefficient than that of the vanilla Q-learning by setting a small $\beta$, so that MomentumQ can enjoy a better convergence rate.

In addition, one can also observe that $\alpha,\beta$ control a set of tradeoffs. First, while smaller $\alpha$ yields a smaller convergence error, it also slows down the convergence rate. As for $\beta$, although smaller $\beta$ yields a smaller coefficient in the dominant term, it can also slow down the convergence rate because $b_i$ in the first term needs to be small. 

Next, we seek to remove the convergence error and balance the tradeoff caused by the choice of $a_k$. To this end, we can choose a diminishing learning rate and obtain the following theorem.
\begin{theorem}\label{thm:acclinearNoniidDimin}
    (MomentumQ with diminishing learning rate) Suppose that Assumptions \ref{asp:boundedPhi}-\ref{asp:markov} hold and fix $\kappa>0$. Let $a_k=\frac{\alpha}{\sqrt{k}}, b_k+c_k=\beta\lambda^k$ with $\alpha>0, \beta,\lambda\in (0,1)$. After running $T$ steps of Algorithm \ref{alg:linearPAQL} under Markovian sampling, we take the output $\theta_{\text{out}} = \frac{1}{T}\sum_{k=1}^T \theta_k$ and have
    \begin{equation}\nonumber
    \begin{aligned}
    \mathbb{E}\norm{\theta_{\text{out}}-\theta^\star}_2^2
    \leq& \frac{ D_{\max}^2/\alpha + 30\alpha G_{\max}^2 + 16\tau^*\alpha\eta_2}{2\delta\sqrt{T}} + \frac{8 D_{\max}G_{\max}\kappa}{\delta}\\
    &+ \frac{1}{T}\left[\frac{5\beta D_{\max}^2}{2\alpha\delta(1-\lambda)^2} + \frac{D_{\max}G_{\max}\beta \lambda + 4\tau^*\eta_1\beta\lambda}{\delta(1-\lambda)} \right],
    \end{aligned}
\end{equation}
    where $\eta_1,\eta_2$ are defined in Lemma \ref{lem:bias}.
\end{theorem}
In Theorem \ref{thm:acclinearNoniidDimin}, if we choose $\kappa = \alpha_k = \alpha/\sqrt{T}$, then the mixing time $\tau^*=\mathcal{O}(\log T)$. Thus, MomentumQ converges to the global optimum at a rate of $\mathcal{O}(\log T/\sqrt{T})$ under a diminishing learning rate.

\subsection{Tabular MomentumQ}
In this subsection, we provide the finite-sample analysis for tabular MomentumQ as listed in~\Cref{alg:AQL}. As we mention in Section \ref{sec:momentumQ}, MomentumQ combines different types of momentum terms dynamically. This requires substantial new technical developments here in the convergence analysis.  

We assume that the state space $\mathcal{X}$ and the action space $\mathcal{U}$ are finite with cardinalities $|\mathcal{X}|$ and $|\mathcal{U}|$, respectively. We denote $n=|\mathcal{X}|\cdot|\mathcal{U}|$. We also need the following assumption in our analysis.
\begin{assumption}\label{asp:boundQ}
    The Q-function is uniformly bounded throughout the learning process. That is, $\exists V_{\max}$, such that $\norm{Q_k}\leq V_{\max},\forall k\geq0$. 
\end{assumption}
Note that it is nontrivial to show the boundedness of the proposed iteration scheme. In fact, it is usually assumed for proving convergence of many such complicated stochastic approximation algorithms~\citep{kushner2003stochastic}. Alternatively, one can extend the ODE method~\citep{borkar2000ode} considerably to show the boundedness, which we left for our future work. 

To facilitate the analysis, we rewrite~\eqref{eq: acc tabular Q-learning} in a more compact form as
\begin{equation}\label{eq:compactAQL}
    \begin{aligned}
    Q_{k+1}=  (1\!-\!a_{k})Q_{k}\!+\!\left[b_{k}(1-a_{k})\!+\!c_{k}\right](Q_{k}-Q_{k-1})  +a_{k}\left[(1+b_{k})\hat{\mathcal{T}}_{k}Q_{k}-b_{k}\hat{\mathcal{T}}_{k}Q_{k-1}\right].
    \end{aligned}
\end{equation}
Our analysis first bounds the errors of approximating the exact Bellman operator $\mathcal{T}$ with empirical Bellman operators $\hat{\mathcal{T}}_k$. For convenience, we denote the $\hat{\mathcal{T}}_k$ terms in~\eqref{eq:compactAQL} by
\begin{equation} 
    \mathcal{D}_k\left[Q_{k},Q_{k-1}\right]:=(1+b_k)\hat{\mathcal{T}}_{k}Q_k-b_k\hat{\mathcal{T}}_{k}Q_{k-1}, \label{eq:Dk}
\end{equation}
for all $k\geq 0$. Note that $\mathcal{D}_{k}$ is a function of all samples $\{y_1,y_2,\cdots, y_k\}$ for all state-action pairs $(x,u)$ up to round $k$. Let $\mathcal{F}_k$ denote the filtration generated by the sequence of these random variables $\{y_1,y_2,\cdots, y_k\}$. We see that $\mathcal{D}_{k}\in \mathcal{F}_{k}$ and $Q_{k+1}\in\mathcal{F}_k$.
Then if we define $\mathcal{D}\left[Q_{k},Q_{k-1}\right]$ as the conditional expectation of $\mathcal{D}_k\left[Q_{k},Q_{k-1}\right]$ given $\mathcal{F}_{k-1}$, we obtain by the definition of $\mathcal{T}$ that
\begin{align}
\mathcal{D}\left[Q_{k},Q_{k-1}\right]
:= \mathbb{E}_{P}\left(\mathcal{D}_{k}\left[Q_{k},Q_{k-1}\right]|\mathcal{F}_{k-1}\right) = (1+b_k)\mathcal{T}Q_{k} -b_k\mathcal{T}Q_{k-1}.\label{eq:exactD}\nonumber 
\end{align}

Now define the error between $\mathcal{D}_k$ and $\mathcal{D}$ as follows:
\begin{equation}\label{eq:epsilonk}
    \epsilon_{k}:=\mathcal{D}\left[Q_{k},Q_{k-1}\right]-\mathcal{D}_{k}\left[Q_{k},Q_{k-1}\right].
\end{equation}
Clearly $\mathbb{E}_{P}\left(\epsilon_{k}|\mathcal{F}_{k-1}\right)=0$.
This shows that $\forall(x,u)\in\mathcal{X}\times U(x)$, the sequence
of the estimation errors $\left\{ \epsilon_{k}(x,u)\right\} _{k=0}^T$
is a martingale difference sequence with respect to the filtration $\mathcal{F}_{k}$.
In other words, if we denote 
\begin{equation}\label{eq:Ek}
    E_{k}(x,u):=\sum_{j=0}^{k}\epsilon_{j}(x,u),
\end{equation}
then $E_k$ is a martingale sequence with respect to $\mathcal{F}_{k},$
$\forall(x,u)\in\mathcal{X}\times U(x)$ and $\forall k\geq0$. 


The following proposition provides the uniform bounds of $\mathcal{D}_k$ and $\epsilon_k$. 
\begin{proposition}\label{lem:boundDk}
Suppose Assumption~\ref{asp:boundQ} holds. Consider MomentumQ as in Algorithm~\ref{alg:AQL}. Then the terms $\mathcal{D}_{k}\left[Q_{k},Q_{k-1}\right]$ defined in \eqref{eq:Dk} and $\epsilon_{k}$ in \eqref{eq:epsilonk}
are uniformly bounded for all $k\geq0$. Specifically, $\exists \bar{D}>0$, s.t. $\norm{\mathcal{D}_{k}[Q_{k},Q_{k-1}]}\leq \bar{D}$,$\norm{\epsilon_{k}}\leq 2\bar{D}, \forall k\geq 0$.
\end{proposition}



The uniform bounds proved in Proposition~\ref{lem:boundDk} are critical in the derivation of the main theorem below.
\begin{theorem}\label{thm:main}
    Suppose Assumption~\ref{asp:boundQ} holds. Consider Algorithm~\ref{alg:AQL} where $m\geq 1/\gamma$. Then, with probability at least $1-\delta$, the output of MomentumQ satisfies for $T>m$:
    \begin{equation}\label{eq:thm}
    \begin{aligned}
        \left\Vert Q^{\star}\!-\!Q_{T}\right\Vert\!\leq\!\frac{\tilde{h}V_{\max} \!+\!\bar{D}\sqrt{8(T-\lfloor m \rfloor -1)\log \frac{2n}{\delta}}}{T(1\!-\!\gamma)},
        \end{aligned}
    \end{equation}
    where $\tilde{h}=2\gamma(m+\lfloor m \rfloor+2)+2$, $\bar{D}$ is specified in Proposition~\ref{lem:boundDk}, and $\lfloor m \rfloor$ denotes the largest integer that does not exceed $m$.
\end{theorem}
\begin{proof}[Proof Sketch for Theorem \ref{thm:main}] The proof of Theorem~\ref{thm:main} is facilitated by several lemmas proved in the appendix. The sketch of the proof is as follows. We first derive the formula of $Q_{k+1}$ in terms of the exact Bellman operator $\mathcal{T}$ and the accumulated approximation error $E_k$ (Lemma~\ref{lem:dynQ}). Then we bound the learning error $\norm{Q_k-Q^{\star}}$ by some constants and the sum of discounted errors $\norm{E_k}$, where we used the contraction of $\mathcal{T}$, the boundedness of $Q_k$, and the assumption $m\geq\frac{1}{\gamma}$ (Lemma~\ref{lem:errorProp}). Finally, we bound the martingale error terms probabilistically using the maximal Hoeffding-Azuma inequality (Lemma~\ref{lem:MHAineq}) and obtain the finite time convergence error in~\eqref{eq:thm}.
\end{proof}

Since $y_k$, $k=0,1,2,...$ are independently sampled, using the second Borel–Cantelli lemma, we immediately have the following corollary.
\begin{corollary}\label{col:main}
$Q_k$ converges to $Q^\star$ almost surely at a rate of at least $\mathcal{O}(\frac{\sqrt{(T-\lfloor m \rfloor -1)\log\frac{2n}{\delta}}}{(1-\gamma)^{2}T})$.
\end{corollary}
This rate is slightly better than $\mathcal{O}(\frac{\sqrt{\log\frac{2n}{\delta}}}{(1-\gamma)^{2}\sqrt{T}})$ of SpeedyQ due to the presence of $m>1$.


\section{Experiments}\label{sec:exp} 

We evaluate the performance of the proposed MomentumQ and compare it with other related Q-learning algorithms over a series of FrozenLake games (see Appendix \ref{app:frozen} for further specifications of the FrozonLake problem). We present the empirical results for tabular MomentumQ and MomentumQ with linear function approximation in Sections~\ref{subsec:taql exp} and~\ref{subsec:paql exp}, respectively.


\subsection{Experiments on Tabular MomentumQ}\label{subsec:taql exp}

We compare our MomentumQ with two other existing momentum-based Q-learning algorithms: SpeedyQ proposed in~\cite[Algorithm 1]{azar2011speedy} and the Nesterov stochastic approximation (NeSA) algorithm proposed in \cite[eq.\ (5) with $\zeta=0.1$]{devraj2019nesa}. We also include the vanilla Q-learning algorithm in our comparison.

The experimental settings in this section are consistent with those of MomentumQ in Algorithm~\ref{alg:AQL} and SpeedyQ in~\cite[Algorithm 1]{azar2011speedy}. Thus the numerical results should be able to give a convincing comparison between two algorithms. It is worth mentioning that the tabular MomentumQ has an additional hyperparameter $m$ that can take a wide range of values (recall $m\geq\frac{1}{\gamma}$). We experiment with several different $m$'s. For relatively large $m$ values (e.g., when $m>10$), the learning rates are shifted to step from $1/(m+1)$, that is, $\alpha_k = 1/(m+k+1)$, for $k=0,1,2,...$. This is to avoid the large errors accumulated from initial iterations when $b_k<0$, which are reflected in the constants in \eqref{eq:thm}. Note that this shift does not change the obtained theoretical order of the convergence rate. We observe stable and often times better performance in convergence across different tests, which also aligns with the theoretical analysis. 

Considering the randomness embedded in MDP of both FrozenLake games, we evaluate the performance of each algorithm with 20 different random seeds and then illustrate the average loss and standard deviation in Fig.~\ref{fig: frozenlake comparison} and Fig.~\ref{fig: frozenlake8 comparison}. For evaluation purpose, we have access to the true transition probability, and can find the ground truth optimal Q-function $Q^{\star}$ using dynamic programming. In both games, the loss at step $k$ is then defined as $\norm{Q_k - Q^\star}$. It can be seen from the results that MomentumQ with various choices of $m$ all can converge faster than the vanilla tabular Q-learning and Speedy Q-learning. It is showing competitive performance against NeSA with smaller variance presented. Note that the high variance observed in the NeSA training aligns with the previous reported results from~\cite{devraj2019nesa} under different tasks.

\begin{figure}[h]
\begin{subfigure}{.99\textwidth}
  \centering
  \includegraphics[trim={3cm 0 4cm 2cm},clip,scale=0.15]{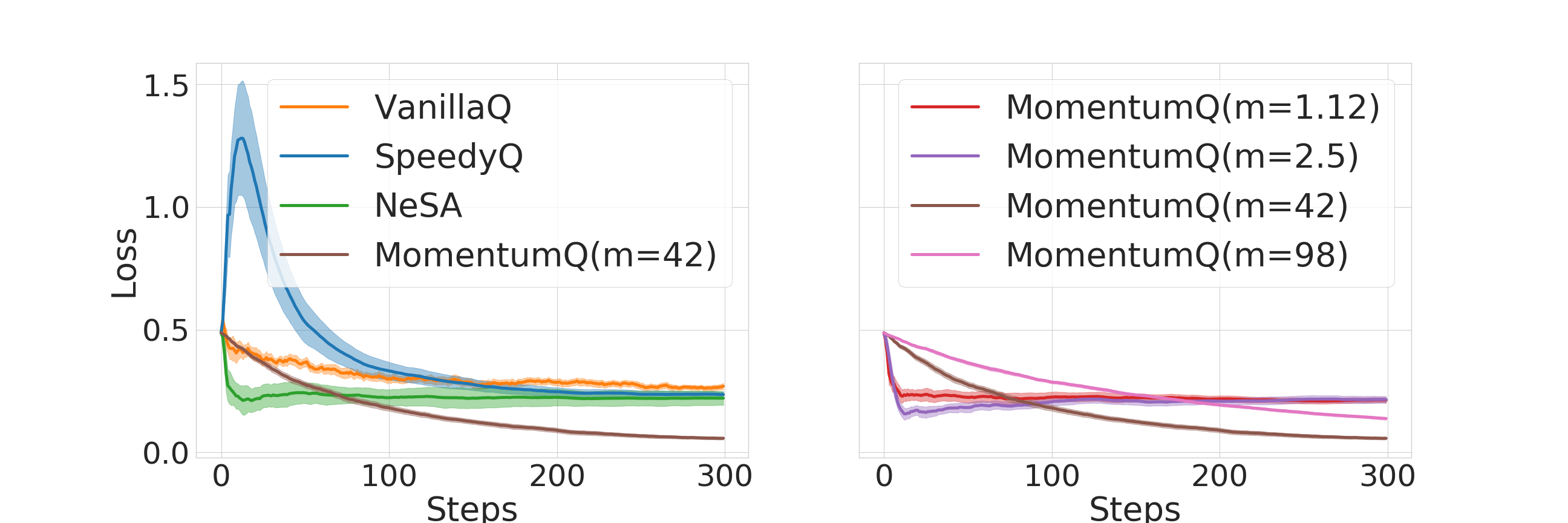}
  \caption{FrozenLake-$4 \times 4$}
  \label{fig: frozenlake comparison}
\end{subfigure}
\begin{subfigure}{.99\textwidth}
  \centering
  \includegraphics[trim={3cm 0 4cm 2cm},clip,scale=0.15]{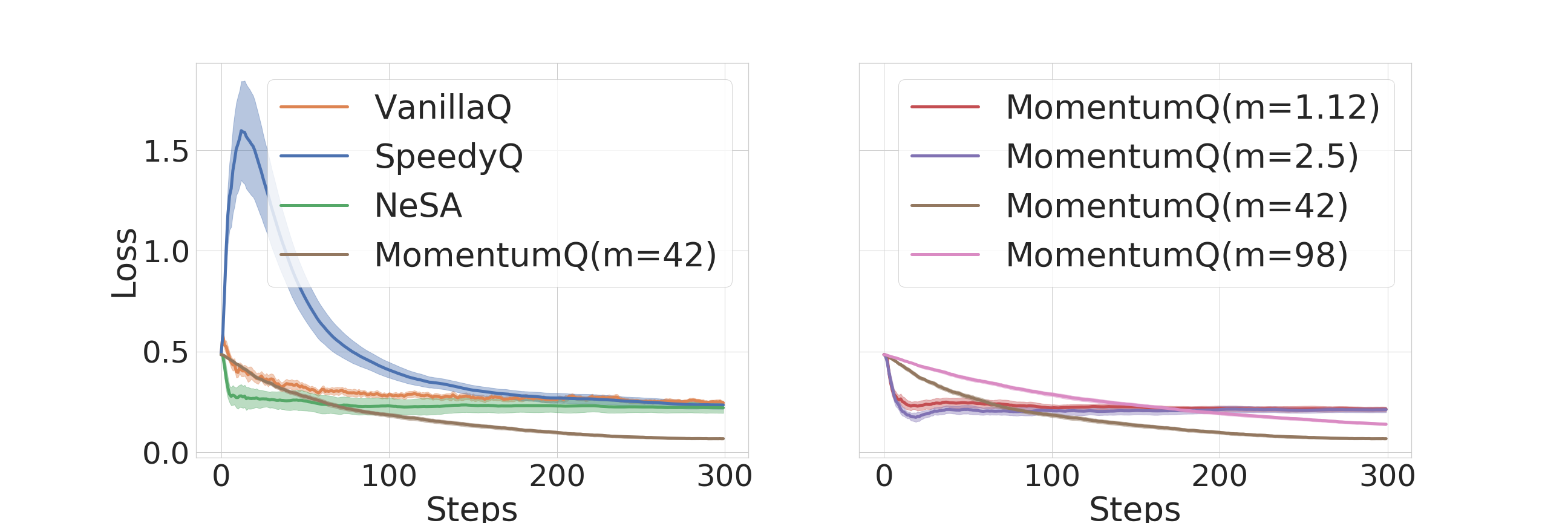}
  \caption{FrozenLake-$8 \times 8$}
  \label{fig: frozenlake8 comparison}
\end{subfigure}
\caption{Comparing MomentumQ with NeSA, SpeedQ, and VanillaQ.}
\label{fig:aql comp}
\vspace{-0.1in}
\end{figure}

\subsection{Experiments on MomentumQ with Function Approximation}\label{subsec:paql exp}

We adopt the FrozenLake-$128 \times 128$ as the benchmark task to evaluate the performance of MomentumQ with linear function approximation and compare it with the vanilla Q-learning (referred to as VanillaQ). Both algorithms are evaluated with different learning rate schemes (constant \& diminishing stepsize), as well as different sampling strategies (i.i.d.\ and Markovian). We note that SpeedyQ and NeSA have been proposed in the literature only for the tabular setting and are thus not included here for comparison.

Note that the i.i.d.\ sampling is an ideal assumption and cannot be satisfied perfectly in practice. For our implementation, we perform i.i.d.\ sampling strategy in a similar fashion to the experience replay~\citep{mnih2013playing} typically used for DQN training. A data buffer, referred to as the experience, is accumulated with data points collected across multiple training steps in the past. At each training step, the training data is then randomly uniformly sampled from the data buffer. In contrast, the Markovian sampling takes the training samples in an ``on-policy" manner where the collected data points are fed in to the Q-learning process right after.

At step $k$, the performance of the algorithm is evaluated through the total return of 150 rounds of trials. Similarly to the tabular setup, we execute each algorithm 20 times with different random seeds and illustrate the average return and standard deviation in Fig~\ref{fig: iid comparison} with i.i.d.\ sampling and Fig~\ref{fig: noniid comparison} with Markovian sampling.

Overall, the MomentumQ algorithm has exhibited superior performance than the vanilla Q-learning. In particular, training with i.i.d.\ sampling is significantly faster than the Markovian sampling, which can be also expected from our theoretical results. Within the same sampling strategy, MomentumQ is also faster in convergence than the vanilla Q-learning with the same learning rate scheme.

\begin{figure}[h]
\begin{subfigure}{.49\textwidth}
  \centering
  \includegraphics[trim={4cm 0 6cm 2cm},clip,scale=0.12]{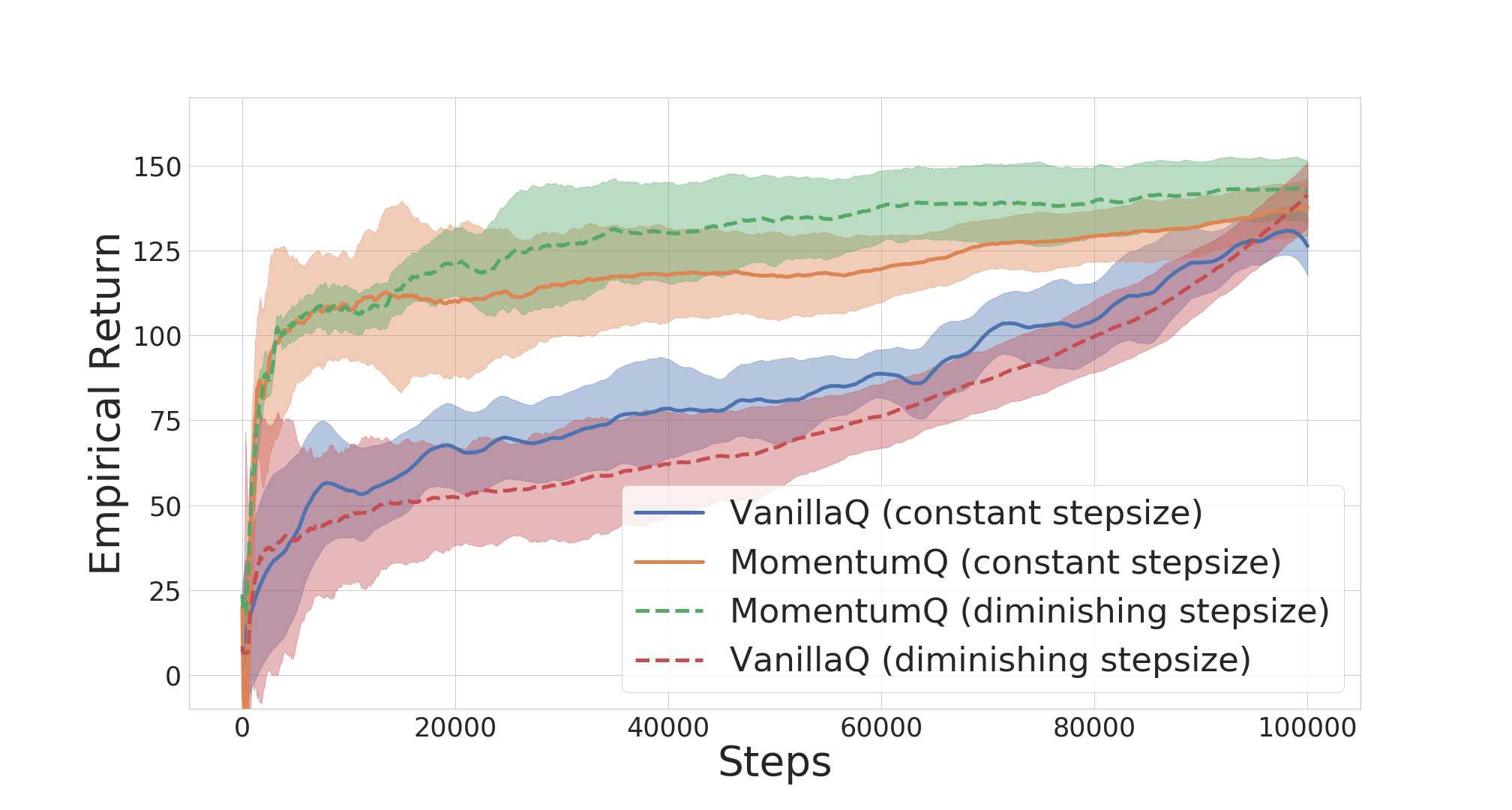}
  \caption{i.i.d sampling}
  \label{fig: iid comparison}
\end{subfigure}
\begin{subfigure}{.49\textwidth}
  \centering
  \includegraphics[trim={4cm 0 6cm 2cm},clip,scale=0.12]{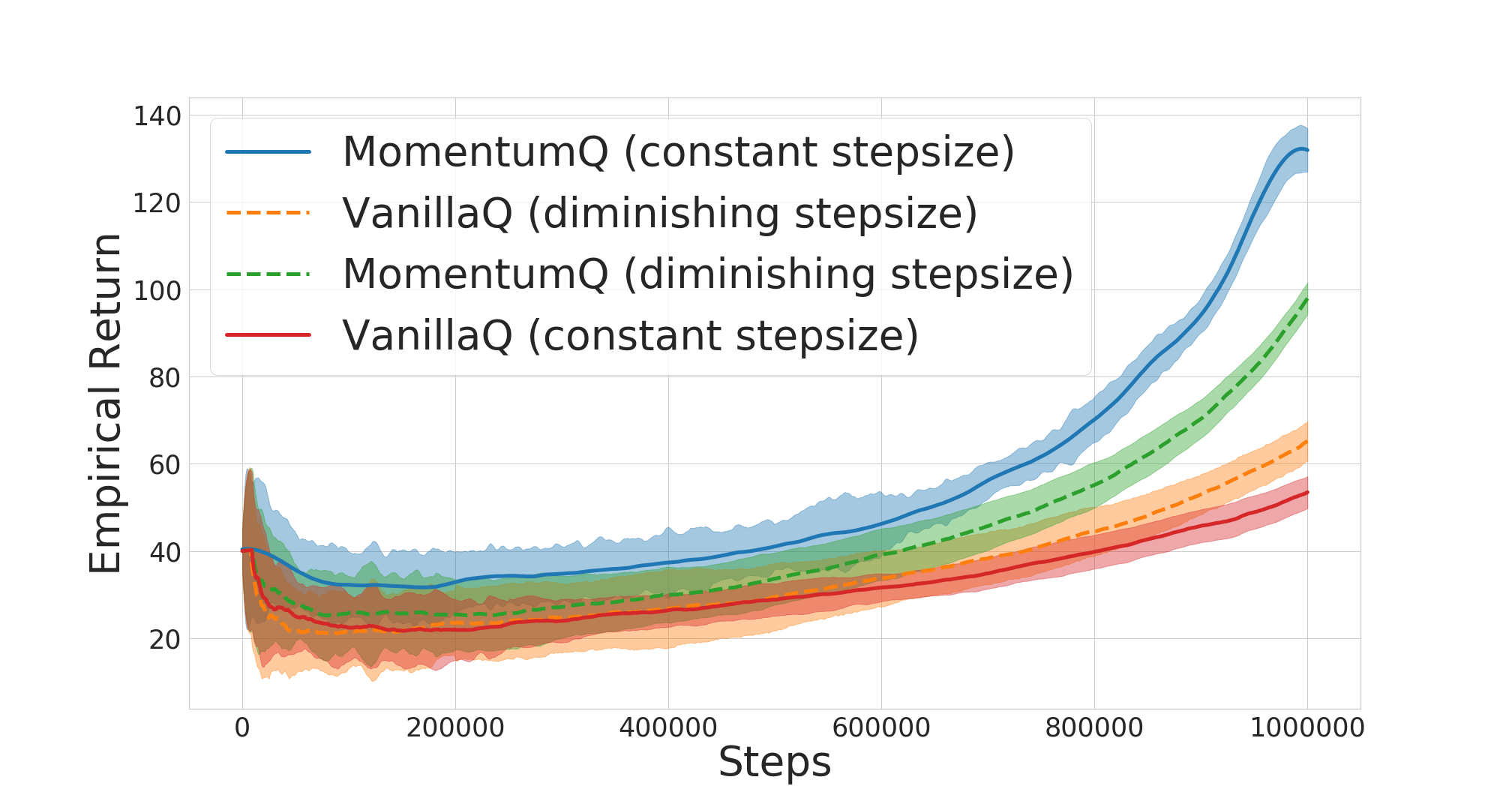}
  \caption{Markovian\ sampling}
  \label{fig: noniid comparison}
\end{subfigure}
\caption{Comparison of MomentumQ with VanillaQ in the FrozenLake-$128 \times 128$ task with various learning rate schemes and sampling strategies.}
\label{fig:aql comp approx}
\vspace{-0.1in}
\end{figure}

\section{Conclusion}

We proposed new momentum-based Q-learning algorithms for both the tabular and linear function approximation cases, which are respectively applicable to finite and continuous state-action spaces. We further characterized the convergence rate for these algorithms, and showed that they converge faster than the SpeedQ and vanilla Q-learning algorithms. We empirically evaluated the algorithms and verified that the proposed algorithms can accelerate the convergence in comparison to vanilla Q-learning on various challenging tasks under both tabular and parametric Q-learning settings.

\section*{Acknowledgements}
The work was supported in part by the U.S. National Science Foundation under Grants CCF-1761506, ECCS-1818904, CCF-1909291 and CCF-1900145, and the startup fund of the Southern University of Science and Technology (SUSTech), China.






\appendix
\vspace{1cm}
\noindent\textbf{\Large Appendices}

\section{Specifications of FrozonLake Problem}\label{app:frozen}

\begin{figure}[tb]
    \centering
    \includegraphics{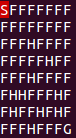}
    \caption{The FrozenLake-$8 \times 8$ task environment.}
    \label{fig:env}
\end{figure}

FrozenLake is a classic benchmark problem for Q-learning, in which an agent controls the movement of a character in an $n \times n$ grid world. Some tiles of the grid are walkable, and others lead to the agent falling into the water. Additionally, the movement direction of the agent is uncertain and only partially depends on the chosen direction. The agent is rewarded for finding a feasible path to a goal tile. As shown in Figure~\ref{fig:env} with a Frozenlake-$8 \times 8$ task, ``S" is the safe starting point, ``F" is the safe frozen surface, ``H" stands for the hole that terminates the game, and ``G" is the target state that comes with an immediate reward of 1. This forms a problem with the state-space size $n^2$, the action-space size $4$ and the reward space $R = \{0, 1\}$. For tabular Q-learning algorithms with finite state-action problems of relatively small dimensions, FrozenLake-$4 \times 4$ and FrozenLake-$8 \times 8$ are two typical benchmark tasks. As the grid world becomes large, e.g., FrozenLake-$128 \times 128$, Q-learning with linear function approximation is then adopted to solve the problem.

\section{Proof of Lemma \ref{lem:bias}}

We bound the expectation of bias via constructing a new Markov chain and applying some techniques from information theory. Before deriving the bound, we first introduce some technical lemmas.

\begin{lemma}\label{lem:boundedGra}
Suppose Assumptions \ref{asp:boundedPhi} and \ref{asp:boundedDomain} hold. Then for $g_k$ defined in \eqref{eq:accLinear}, we have $\norm{g_k}_2 \leq G_{\max}$ for all $k$, where $G_{\max} = 2D_{\max} + R_{\max}$.
\end{lemma}
\begin{proof}
Following from the definition of $g_k$ and the assumptions that $\norm{\Phi(x,u)}_2\leq 1, \norm{\theta}_2\leq D_{\max}$, and $\norm{R(x,u)}_2\leq R_{\max}$, we have
\begin{align*}
\norm{g_k}_2 =& \norm{(\Phi(x_k,u_k)^T\theta_k - R(x_k,u_k)- \gamma \underset{u'\in U(x_{k+1})}{\max} \Phi(x_{k+1},u')^T\theta_k)\Phi(x_k,u_k)}_2\\
\leq& \norm{\Phi(x_k,u_k)^T\theta_k}_2 + \norm{R(x_k,u_k)}_2 + \underset{u'\in U(x_{k+1})}{\max}\norm{\Phi(x_{k+1},u')^T\theta_k}_2\\
\leq& 2D_{\max} + R_{\max},
\end{align*}
where we use Cauchy-Schwartz inequality and the triangle inequality.
\end{proof}

For notational simplicity, throughout this section we use $O=(x,u,x')$ to denote the sample tuple and $O_k=(x_k,u_k,x_{k+1})$ to denote the sample tuple at time $k$.

\begin{lemma}\label{lem:biasLip}
Let $\xi(\theta;O) := (g(\theta;O) - \bar g(\theta))^T(\theta - \theta^\star)$. Then $\xi(\theta;O)$ is uniformly bounded by
$$ \lvert \xi(\theta;O) \rvert \leq 2D_{\max}G_{\max}, \quad\forall \theta\in\mathcal{B}, $$
and it is Lipschitz continuous with
$$ \lvert \xi(\theta;O)-\xi(\theta';O) \rvert \leq 2((1+\gamma)D_{\max}+G_{\max})\norm{\theta-\theta'}_2, \quad\forall \theta,\theta'\in\mathcal{B}. $$
\end{lemma}
\begin{proof}
The first statement is straightforward based on Assumption \ref{asp:boundedDomain} and Lemma \ref{lem:boundedGra}. That is,
$$ \lvert \xi(\theta;O) \rvert \leq \norm{g(\theta;O) - \bar g(\theta)}_2\norm{\theta - \theta^\star}_2\leq 2D_{\max}G_{\max}.
$$

Next to prove the Lipschitz condition, we first prove the Lipschitz condition of $g(\theta;O_k)$ with respect to $\theta$.
\begin{align*}
    \norm{g(\theta;O) - g(\theta';O)}_2 \overset{\text{(i)}}{\leq}&\lvert \Phi(x,u)^T(\theta - \theta') + \gamma\underset{u'\in U(x')}{\max}\Phi(x',u')^T\theta' - \gamma\underset{u'\in U(x')}{\max}\Phi(x',u')^T\theta \rvert\\
    \overset{\text{(ii)}}{\leq}& \lvert{\Phi(x,u)^T(\theta - \theta')} \rvert + \lvert \gamma\underset{u'\in U(x')}{\max}\Phi(x',u')^T\theta' - \gamma\underset{u'\in U(x')}{\max}\Phi(x',u')^T\theta \rvert,
\end{align*}
where (i) follows from Cauchy-Schwartz inequality and the assumption $\norm{\Phi}_2\leq1$, and (ii) follows from the triangle inequality. 

Now we consider two cases. If the item in the second norm of (ii) is non-negative, we let $u^\star = \underset{u'\in U(x')}{\arg\max}\Phi(x',u')^T\theta'$. Then $\underset{u'\in U(x')}{\max}\Phi(x',u')^T\theta \geq \Phi(x',u^\star)^T\theta$. Thus, we continue to bound the above inequality as
    \begin{align}
    \norm{g(\theta;O)-g(\theta';O)}_2\leq & \lvert{\Phi(x,u)^T(\theta - \theta')} \rvert + \gamma\Phi(x',u^\star)^T(\theta' - \theta) \nonumber\\ &
    \lvert{\Phi(x,u)^T(\theta - \theta')} \rvert + \gamma\lvert{\Phi(x',u^\star)^T(\theta - \theta')} \rvert \label{eq:lembiaspf1} 
    \end{align}
Similarly, if this item is negative, we let $u^\star = \underset{u'\in U(x')}{\arg\max}\Phi(x',u')^T\theta$. Then $\underset{u'\in U(x')}{\max}\Phi(x',u')^T\theta' \geq \Phi(x',u^\star)^T\theta'$. Thus, we have 
    \begin{align}
    \norm{g(\theta;O)-g(\theta';O)}_2\leq & \lvert{\Phi(x,u)^T(\theta - \theta')} \rvert + \gamma\Phi(x',u^\star)^T(\theta - \theta') \nonumber\\ &
    \lvert{\Phi(x,u)^T(\theta - \theta')} \rvert + \gamma\lvert{\Phi(x',u^\star)^T(\theta - \theta')} \rvert \label{eq:lembiaspf2} 
    \end{align}
Then it follows from \eqref{eq:lembiaspf1} and \eqref{eq:lembiaspf2} that 
\[
\norm{g(\theta;O)-g(\theta';O)}_2\leq (1+\gamma)\norm{\theta-\theta'}.
\]
Similarly, we obtain the same result for $\bar g(\theta)$ as follows.
$$ \norm{\bar g(\theta) - \bar g(\theta')}_2 \leq \underset{\mu}{\mathbb{E}} \norm{g_k(\theta) - g_k(\theta')}_2 \leq (1+\gamma)\norm{\theta - \theta'}_2.
$$

Then we focus on obtaining the second statement,
\begin{align*}
     &\lvert \xi(\theta;O)-\xi(\theta';O) \rvert\\
     &\quad= \lvert (g(\theta;O) - \bar g(\theta))^T(\theta - \theta^\star) - (g(\theta';O) - \bar g(\theta'))^T(\theta' - \theta^\star) \rvert\\
     &\quad\leq \norm{g(\theta;O) - \bar g(\theta)}_2\norm{\theta - \theta'}_2 + \norm{\theta' - \theta^\star}_2\norm{(g(\theta;O) - \bar g(\theta)) - (g(\theta';O) - \bar g(\theta'))}_2\\
     &\quad\overset{\text{(i)}}{\leq} 2G_{\max}\norm{\theta - \theta'}_2 + D_{\max} \norm{(g(\theta;O) - g(\theta';O)) - (\bar g(\theta) - \bar g(\theta'))}_2\\
     &\quad\overset{\text{(ii)}}{\leq} 2G_{\max}\norm{\theta - \theta'}_2 + 2D_{\max}(1+\gamma)\norm{\theta - \theta'}_2\\
     &\quad= 2((1+\gamma)D_{\max}+G_{\max})\norm{\theta-\theta'}_2,
\end{align*}
where (i) follows from Assumption \ref{asp:boundedDomain} and Lemma \ref{lem:boundedGra}, and (ii) follows from triangle inequality and \eqref{eq:lembiaspf1}.
\end{proof}

We use $X\rightarrow Z\rightarrow Y$ to indicate that the random variable $X$ and $Y$ are independent conditioned on $Z$.
\begin{lemma}\label{lem:biasMarkov}
\citep[Lemma 9]{bhandari2018finite} Consider two random variables $X$ and $Y$ such that 
\begin{equation}\label{eq:lemMar}
    X\rightarrow x_k \rightarrow x_{k+\tau} \rightarrow Y,
\end{equation} 
for fixed $ k $ and $\tau>0$. Suppose Assumption \ref{asp:markov} holds. Let $X',Y'$ are independent copies drawn from the marginal distributions of $X$ and $Y$, that is $\mathbb{P}(X'=\cdot,Y'=\cdot) = \mathbb{P}(X=\cdot)  \mathbb{P}(Y=\cdot)$. Then, 
for any bounded $v$, we have
$$ \lvert \mathbb{E}[v(X,Y)] - \mathbb{E}[v(X',Y')] \rvert \leq 2 \norm{v}_{\infty}(\sigma\rho^{\tau}).
$$
\end{lemma}

We continue the proof of Lemma \ref{lem:bias}. We first develop the connection between $\xi(\theta_k;O_k)$ and $\xi(\theta_{k-\tau};O_k)$ via Lemma \ref{lem:biasLip}. To do so, we first observe that 
\begin{align*} 
\norm{\theta_{i+1} - \theta_i}_2 =& \norm{\beta_i(\theta_i - \theta_{i-1}) + a_i(1+b_i)g_i + a_ib_ig_{i-1} }_2\\
\overset{\text{(i)}}{\leq}& \norm{\beta_i(\theta_i - \theta_{i-1})}_2 + \norm{a_i(1+b_i)g_i}_2 + \norm{a_ib_ig_{i-1}}_2\\
\overset{\text{(ii)}}{\leq}& D_{\max}\beta_i + 3G_{\max}a_i,
\end{align*}
where (i) follows from the triangle inequality and (ii) from the Assumptions \ref{asp:boundedDomain} and \ref{lem:boundedGra} and the fact $b_i < 1$.
Then we have
$$ \norm{\theta_k - \theta_{k-\tau}}_2 \leq \sum_{i=k-\tau}^{k-1} \norm{\theta_{i+1} - \theta_i}_2\leq D_{\max}\sum_{i=k-\tau}^{k-1} \beta_i + 3G_{\max}\sum_{i=k-\tau}^{k-1} a_i.
$$
Thus, we can relate $\xi(\theta_k;O_k)$ and $\xi(\theta_{k-\tau};O_k)$ by using the Lipschitz property established in Lemma \ref{lem:biasLip} as follows:
\begin{align} 
    \xi(\theta_k;O_k) - \xi(\theta_{k-\tau};O_k) \leq& \lvert \xi(\theta_k;O_k) - \xi(\theta_{k-\tau};O_k) \rvert \nonumber\\
    \leq& 2((1+\gamma)D_{\max}+G_{\max})\norm{\theta_k-\theta_{k-\tau}}_2 \nonumber\\
    \leq& 2((1+\gamma)D_{\max}+G_{\max})\left( D_{\max}\sum_{i=k-\tau}^{k-1} \beta_i + 3G_{\max}\sum_{i=k-\tau}^{k-1} a_i \right). \label{eq:lembiaspfstep1}
\end{align}

Next, we bound $\mathbb{E}[\xi(\theta_{k-\tau};O_k)]$ using Lemma \ref{lem:biasMarkov}. 
Observe that given any deterministic $\theta\in\mathcal{B}$, we have 
$$ \mathbb{E}[\xi(\theta; O_k)] = (\mathbb{E}[g(\theta; O_k)] - \bar g(\theta))^T(\theta - \theta^\star) = 0.
$$
Since $\theta_0$ is a fixed constant, we have $\mathbb{E}[\xi(\theta_0, O_k)] = 0$.
Now we are ready to bound $ \mathbb{E}[\xi(\theta_{k-\tau},O_k)]$ via Lemma \ref{lem:biasMarkov} by constructing a random process satisfying \eqref{eq:lemMar}. To do so, consider random variables $\theta'_{k-\tau}$ and $O'_k$ drawn independently from the marginal distribution of $\theta_{k-\tau}$ and $O_k$, so that $\mathbb{P}(\theta'_{k-\tau}=\cdot,O'_k=\cdot) = \mathbb{P}(\theta_{k-\tau}=\cdot)  \mathbb{P}(O_k=\cdot)$. We further obtain $\mathbb{E}[\xi(\theta'_{k-\tau},O'_k)]=\mathbb{E}[\mathbb{E}[\xi(\theta'_{k-\tau},O'_k)|\theta'_{k-\tau}]] = 0 $ since $\theta'_{k-\tau}$ and $O'_k$ are independent. Combining Lemmas \ref{lem:biasLip} and \ref{lem:biasMarkov}, we have
\begin{equation}\label{eq:lembiaspfstep2}
    \mathbb{E}[\xi(\theta_{k-\tau},O_k)]\leq 2 (2D_{\max}G_{\max}) (\sigma\rho^{\tau}).
\end{equation}

Finally, we are ready to bound the bias. 
We first take expectation for both sides of \eqref{eq:lembiaspfstep1} and obtain
$$ \mathbb{E}[\xi(\theta_k;O_k)] \leq \mathbb{E}[\xi(\theta_{k-\tau};O_k)] + 2((1+\gamma)D_{\max}+G_{\max})\left( D_{\max}\sum_{i=k-\tau}^{k-1} \beta_i + 3G_{\max}\sum_{i=k-\tau}^{k-1} a_i \right).
$$
When $k \leq \tau^{mix}(\kappa)$, we choose $\tau = k$ and have
$$\begin{aligned}
\mathbb{E}[\xi(\theta_k;O_k)] \leq& \mathbb{E}[\xi(\theta_{0};O_k)] + 2((1+\gamma)D_{\max}+G_{\max})\left( D_{\max}\sum_{i=0}^{k-1} \beta_i + 3G_{\max}\sum_{i=0}^{k-1} a_i \right)\\
=& 2((1+\gamma)D_{\max}+G_{\max})\left( D_{\max}\sum_{i=0}^{k-1} \beta_i + 3G_{\max}\sum_{i=0}^{k-1} a_i \right).
\end{aligned}
$$
When $k > \tau^{mix}(\kappa)$, we choose $\tau = \tau^*:=\tau^{mix}(\kappa) $ and have
\begin{align*}
    \mathbb{E}&[\xi(\theta_k;O_k)]\\
    &\quad\leq \mathbb{E}[\xi(\theta_{k-\tau^*};O_k)] + 2((1+\gamma)D_{\max}+G_{\max})\left( D_{\max}\sum_{i=k-\tau^*}^{k-1} \beta_i + 3G_{\max}\sum_{i=k-\tau^*}^{k-1} a_i \right)\\
    &\quad\overset{\text{(i)}}{\leq} 4D_{\max}G_{\max}(\sigma\rho^{\tau^*}) + 2((1+\gamma)D_{\max}+G_{\max})\left( D_{\max}\sum_{i=k-\tau^*}^{k-1} \beta_i + 3G_{\max}\sum_{i=k-\tau^*}^{k-1} a_i \right)\\
    &\quad\overset{\text{(ii)}}{\leq} 4D_{\max}G_{\max}\kappa + 2((1+\gamma)D_{\max}+G_{\max})\left( D_{\max}\sum_{i=k-\tau^*}^{k-1} \beta_i + 3G_{\max}\sum_{i=k-\tau^*}^{k-1} a_i \right)\\
    &\quad\overset{\text{(iii)}}{\leq} 4D_{\max}G_{\max}\kappa + 2((1+\gamma)D_{\max}+G_{\max})\left( D_{\max}\tau^* \beta_{k-\tau^*} + 3G_{\max}\tau^* a_{k-\tau^*} \right),
\end{align*}
where (i) follows from \eqref{eq:lembiaspfstep2}, (ii) follows due to the definition of the mixing time, and (iii) follows because $a_k, \beta_k$ are non-increasing.

\section{Proof of Theorem \ref{thm:acclinearNoniid} }

Recall that MomentumQ with linear function approximation updates as \eqref{eq:accLinear}. Given the unique fixed point $\theta^\star$ and denoting $b_k + c_k = \beta_k$, we have
\begin{align*}
    \norm{\theta_{k+1}-\theta^\star}_2^2 =& \norm{\theta_k - \theta^\star + \beta_k(\theta_k - \theta_{k-1}) -a_k(1+b_k)g_k + a_k b_k g_{k-1}}_2^2\\
    =& \norm{\theta_{k}-\theta^\star}_2^2 + \norm{\beta_k(\theta_k - \theta_{k-1}) -a_k(1+b_k)g_k + a_k b_k g_{k-1}}_2^2\\
    &+2\langle \theta_{k}-\theta^\star, \beta_k(\theta_k - \theta_{k-1}) -a_k(1+b_k)g_k + a_k b_k g_{k-1} \rangle\\
    =& \norm{\theta_{k}-\theta^\star}_2^2 + \norm{\beta_k(\theta_k - \theta_{k-1}) -a_k(1+b_k)g_k + a_k b_k g_{k-1}}_2^2\\
    &+2\langle \theta_{k}-\theta^\star, \beta_k(\theta_k - \theta_{k-1}) + a_k b_k g_{k-1} \rangle - 2a_k(1+b_k)\langle \theta_{k}-\theta^\star, g_k \rangle.
\end{align*}

Next, taking the expectation over all the randomness up to time step $k$ on both sides, we have
\begin{align}
    \mathbb{E}&\norm{\theta_{k+1}-\theta^\star}_2^2\nonumber\\
    =& \mathbb{E}\norm{\theta_{k}-\theta^\star}_2^2 + \mathbb{E}\norm{\beta_k(\theta_k - \theta_{k-1}) -a_k(1+b_k)g_k + a_k b_k g_{k-1}}_2^2\nonumber\\
    &+2\mathbb{E}\langle \theta_{k}-\theta^\star, \beta_k(\theta_k - \theta_{k-1}) + a_k b_k g_{k-1} \rangle - 2a_k(1+b_k)\mathbb{E}\langle \theta_{k}-\theta^\star, g_k \rangle\nonumber\\
    &+2\mathbb{E}\langle \theta_{k}-\theta^\star, \beta_k(\theta_k - \theta_{k-1}) + a_k b_k g_{k-1} \rangle - 2a_k(1+b_k)\mathbb{E}\langle \theta_{k}-\theta^\star, g_k \rangle \nonumber\\
    \overset{\text{(i)}}{\leq}& \mathbb{E}\norm{\theta_{k}-\theta^\star}_2^2 + \mathbb{E}\norm{\beta_k(\theta_k - \theta_{k-1}) -a_k(1+b_k)g_k + a_k b_k g_{k-1}}_2^2\nonumber\\
    &+2\beta_k\mathbb{E}\norm{\theta_{k}-\theta^\star}_2 \norm{\theta_k - \theta_{k-1}}_2 + 2a_k b_k\mathbb{E}\norm{\theta_{k}-\theta^\star}_2 \norm{g_{k-1}}_2  - 2a_k(1+b_k)\mathbb{E}\langle \theta_{k}-\theta^\star, g_k \rangle \nonumber\\
    \overset{\text{(ii)}}{\leq}& \mathbb{E}\norm{\theta_{k}-\theta^\star}_2^2 + 3\beta_k^2\mathbb{E}\norm{\theta_k - \theta_{k-1}}_2^2 + 3a_k^2(1+b_k)^2\mathbb{E}\norm{g_{k}}_2^2 + 3a_k^2 b_k^2 \mathbb{E}\norm{g_{k-1}}_2^2\nonumber\\
    &+2\beta_k\mathbb{E}\norm{\theta_{k}-\theta^\star}_2 \norm{\theta_k - \theta_{k-1}}_2 + 2a_k b_k\mathbb{E}\norm{\theta_{k} -\theta^\star}_2 \norm{g_{k-1}}_2  - 2a_k(1+b_k)\mathbb{E}\langle \theta_{k}-\theta^\star, g_k \rangle \nonumber\\
    \overset{\text{(iii)}}{\leq}& \mathbb{E}\norm{\theta_{k}-\theta^\star}_2^2 + 3\beta_k^2 D_{\max}^2 + 3a_k^2(1+b_k)^2G_{\max}^2 + 3a_k^2 b_k^2 G_{\max}^2\nonumber\\
    &+2\beta_k D_{\max}^2 + 2a_k b_k D_{\max}G_{\max}  - 2a_k(1+b_k)\mathbb{E}\langle \theta_{k}-\theta^\star, g_k \rangle\nonumber\\
    \overset{\text{(iv)}}{\leq}& \mathbb{E}\norm{\theta_{k}-\theta^\star}_2^2 + 5\beta_k D_{\max}^2 + 15a_k^2G_{\max}^2 + 2a_k b_k D_{\max}G_{\max}  - 2a_k(1+b_k)\mathbb{E}\langle \theta_{k}-\theta^\star, g_k \rangle, \label{eq:pfthm22}
\end{align}
where (i) follows from Cauchy-Schwartz inequality, (ii) holds due to the fact $(x+y+z)^2\leq 3x^2+3y^2+3z^2$, (iii) holds because of the boundedness of the parameter domain in Assumption \ref{asp:boundedDomain} and because of Lemma \ref{lem:boundedGra}, and (iv) follows since $b_k\leq\beta_k<1$.

Since the samples are generated in a non-i.i.d.\ manner, we have
\begin{align}
    \mathbb{E}\left[(\theta_k - \theta^\star)^T g_{k}\right] &= \mathbb{E}\left[(\theta_k - \theta^\star)^T \bar g(\theta_{k})\right] + \mathbb{E}\left[(\theta_k - \theta^\star)^T(g_k - \bar g(\theta_{k}))\right] \nonumber\\
    &= \mathbb{E}\left[(\theta_k - \theta^\star)^T \bar g(\theta_{k})\right] + \mathbb{E} [\xi(\theta_k;O_k)].\label{eq:thmnoniidpf1}
\end{align}
Then, we continue to bound \eqref{eq:pfthm22} and obtain
\begin{align}
    &\mathbb{E}\norm{\theta_{k+1}-\theta^\star}_2^2 \nonumber\\
    &\quad\leq  \mathbb{E}\norm{\theta_{k}-\theta^\star}_2^2 + 5\beta_k D_{\max}^2 + 15a_k^2G_{\max}^2 + 2a_k b_k D_{\max}G_{\max}  - 2a_k(1+b_k)\mathbb{E}\langle \theta_{k}-\theta^\star, g_k \rangle\nonumber\\
    &\quad = \mathbb{E}\norm{\theta_{k}-\theta^\star}_2^2 + 5\beta_k D_{\max}^2 + 15a_k^2G_{\max}^2 + 2a_k b_k D_{\max}G_{\max} \nonumber\\
    &\quad\quad - 2a_k(1+b_k)\mathbb{E}\langle \theta_{k}-\theta^\star, \bar g(\theta_k) \rangle - 2a_k(1+b_k)\mathbb{E} [\xi(\theta_k;O_k)]\nonumber\\
    &\quad\leq \mathbb{E}\norm{\theta_{k}-\theta^\star}_2^2 + 5\beta_k D_{\max}^2 + 15a_k^2G_{\max}^2 + 2a_k b_k D_{\max}G_{\max}  - 2a_k(1+b_k)\delta\mathbb{E}\norm{\theta_{k}-\theta^\star}_2^2\nonumber\\
    &\quad\quad - 2a_k(1+b_k)\mathbb{E} [\xi(\theta_k;O_k)]\nonumber\\
    &\quad= (1 - 2a_k\delta(1+b_k))\mathbb{E}\norm{\theta_{k}-\theta^\star}_2^2 + 5\beta_k D_{\max}^2 + 15a_k^2G_{\max}^2 + 2a_k b_k D_{\max}G_{\max}\nonumber\\
    &\quad\quad - 2a_k(1+b_k)\mathbb{E} [\xi(\theta_k;O_k)],\label{eq:pf_diffstart2}
\end{align}
where the last inequality follows from Assumption \ref{asp:qlearning}.

We consider a constant stepsize $\alpha_k=\alpha$. For notational simplicity, we denote $f_k=5\beta_k D_{\max}^2 + 15a_k^2G_{\max}^2 + 2a_k b_k D_{\max}G_{\max}$, and $\zeta_k=- 2a_k(1+b_k)\mathbb{E} [\xi(\theta_k;O_k)]$. Then for $k>\tau^*$ we have
\begin{align*}
    &\mathbb{E}\norm{\theta_{k+1}-\theta^\star}_2^2 \\
    &\quad\leq (1 - 2\alpha\delta(1+b_k))\mathbb{E}\norm{\theta_{k}-\theta^\star}_2^2 + f_k + \zeta_k\\
    &\quad\leq \dots\\
    &\quad\leq \prod_{i=0}^k (1 - 2\alpha\delta(1+b_i))\norm{\theta_{0}-\theta^\star}_2^2 + \sum_{i=0}^k f_i \prod_{j=i+1}^k (1 - 2\alpha\delta(1+b_j))\\
    &\quad\quad + \sum_{i=\tau^*+1}^k \zeta_i \prod_{j=i+1}^k (1 - 2\alpha\delta(1+b_j)) + \sum_{i=0}^{\tau^*} \zeta_i \prod_{j=i+1}^k (1 - 2\alpha\delta(1+b_j))\\
    &\quad\leq \prod_{i=0}^k (1 - 2\alpha\delta(1+b_i))\norm{\theta_{0}-\theta^\star}_2^2 + \sum_{i=0}^k f_i (1 - 2\alpha\delta)^{k-i}\\
    &\quad\quad + \sum_{i=\tau^*+1}^k \zeta_i (1 - 2\alpha\delta)^{k-i} + \sum_{i=0}^{\tau^*} \zeta_i (1 - 2\alpha\delta)^{k-i},
\end{align*}
where the last inequality follows because $b_k>0,\ \forall k$. 
Further, we bound the term $\sum_{i=0}^k (1 - 2\alpha\delta)^{k-i} f_i$ as
\begin{align}
    &\sum_{i=0}^{k} (1-2\delta \alpha)^{k-i} f_i\nonumber\\
    &= 5 D_{\max}^2\sum_{i=0}^{k} (1-2\delta \alpha)^{k-i}\beta_i + 15\alpha^2G_{\max}^2\sum_{i=0}^{k} (1-2\delta \alpha)^{k-i} + 2\alpha D_{\max}G_{\max}\sum_{i=0}^{k} (1-2\delta \alpha)^{k-i}b_i\nonumber\\
    &\leq 15\alpha^2G_{\max}^2\sum_{i=0}^{k} (1-2\delta \alpha)^{k-i} + (5 D_{\max}^2 + 2\alpha D_{\max}G_{\max} )\sum_{i=0}^{k} (1-2\delta \alpha)^{k-i}\beta_i\nonumber\\
    &\leq \frac{15\alpha G_{\max}^2}{2\delta} + (5 D_{\max}^2 + 2\alpha D_{\max}G_{\max} )\beta(1-2\delta \alpha)^k\sum_{i=0}^{k}\left(\frac{\lambda}{1-2\delta \alpha}\right)^i\nonumber\\
    &\overset{\text{(i)}}{\leq} \frac{15\alpha G_{\max}^2}{2\delta} + (5 D_{\max}^2 + 2\alpha D_{\max}G_{\max} )\beta(1-2\delta \alpha)^k \frac{1}{1-2\delta \alpha - \lambda}, \label{eq:pfterm1}
\end{align}
where (i) follows from $\alpha<\frac{1-\lambda}{2\delta}$. 
It remains to bound the last two tail terms.
From Lemma \ref{lem:bias}, we obtain
\begin{equation*}
    \zeta_i=\left\{\begin{aligned}
    &2\alpha(1+b_i)\left(\eta_1\sum_{i=1}^{k-1}\beta_i + \eta_2\sum_{i=1}^{k-1}a_i\right)\leq 4\alpha\left(\eta_1\tau^*\beta + \eta_2\tau^*\alpha\right),\quad i\leq \tau^*;\\
    &4\alpha\left(4D_{\max}G_{\max}\kappa + \eta_1\tau^*\beta_{i-\tau^*} + \eta_2\tau^*\alpha\right),\quad i>\tau^*,
    \end{aligned}
    \right.
\end{equation*}
where $\eta_1 = 2D_{\max}((1+\gamma)D_{\max}+G_{\max}),\eta_2 = 6G_{\max}((1+\gamma)D_{\max}+G_{\max})$. Then we obtain
\begin{align*}
    &\sum_{i=\tau^*+1}^k \zeta_i (1 - 2\alpha\delta)^{k-i} + \sum_{i=0}^{\tau^*} \zeta_i (1 - 2\alpha\delta)^{k-i} \\
    &\quad\leq 4\eta_2\tau^*\alpha^2\sum_{i=0}^{k} (1 - 2\alpha\delta)^{k-i} + 4\alpha\eta_1\tau^*\beta \sum_{i=0}^{\tau^*} (1 - 2\alpha\delta)^{k-i}\\
    &\quad\quad + 16D_{\max}G_{\max}\kappa\alpha \sum_{i=\tau^*+1}^k (1 - 2\alpha\delta)^{k-i} + 4\alpha \eta_1\tau^*  \sum_{i=\tau^*+1}^k \beta_{i-\tau^*}(1 - 2\alpha\delta)^{k-i}\\
    &\quad\leq \frac{2\eta_2\tau^*\alpha}{\delta} + \frac{2\eta_1\tau^*\beta}{\delta} (1 - 2\alpha\delta)^{k-\tau^*} + \frac{8D_{\max}G_{\max}\kappa}{\delta} + 4\alpha\beta \eta_1\tau^*  \sum_{i=\tau^*+1}^k \lambda^{i-\tau^*}(1 - 2\alpha\delta)^{k-i}\\
    &\quad=  \frac{2\eta_2\tau^*\alpha}{\delta} + \frac{2\eta_1\tau^*\beta}{\delta} (1 - 2\alpha\delta)^{k-\tau^*} + \frac{8D_{\max}G_{\max}\kappa}{\delta}\\
    &\quad\quad + 4\alpha\beta \eta_1\tau^* (1 - 2\alpha\delta)^{k-\tau^*} \sum_{i=\tau^*+1}^k \left(\frac{\lambda}{1 - 2\alpha\delta}\right)^{i-\tau^*}\\
    &\quad\leq \frac{2\eta_2\tau^*\alpha}{\delta} + \frac{2\eta_1\tau^*\beta}{\delta} (1 - 2\alpha\delta)^{k-\tau^*} + \frac{8D_{\max}G_{\max}\kappa}{\delta} + \frac{4\alpha\beta \eta_1\tau^*\lambda}{1 - 2\alpha\delta-\lambda} (1 - 2\alpha\delta)^{k-\tau^*},
\end{align*}
where the last inequality follows due to the fact that $\alpha < \frac{1-\lambda}{2\delta}$. Thus, we can conclude that
\begin{align*}
    \mathbb{E}&\norm{\theta_{k+1}-\theta^\star}_2^2 \\
    &\ \leq \prod_{i=0}^k (1 - 2\alpha\delta(1+b_i))\norm{\theta_{0}-\theta^\star}_2^2 + \sum_{i=0}^k f_i (1 - 2\alpha\delta)^{k-i}\\
    &\ \quad + \sum_{i=\tau^*+1}^k \zeta_i (1 - 2\alpha\delta)^{k-i} + \sum_{i=0}^{\tau^*} \zeta_i (1 - 2\alpha\delta)^{k-i}\\
    &\ \leq \prod_{i=0}^k (1 - 2\alpha\delta(1+b_i))\norm{\theta_{0}-\theta^\star}_2^2 + \frac{15\alpha G_{\max}^2}{2\delta} +  \frac{\beta(5 D_{\max}^2 + 2\alpha D_{\max}G_{\max} )(1-2\delta \alpha)^k}{1-2\delta \alpha - \lambda}\\
    &\ \quad + \frac{2\eta_2\tau^*\alpha}{\delta} + \frac{2\eta_1\tau^*\beta}{\delta} (1 - 2\alpha\delta)^{k-\tau^*} + \frac{8D_{\max}G_{\max}\kappa}{\delta} + \frac{4\alpha\beta \eta_1\tau^*\lambda}{1 - 2\alpha\delta-\lambda} (1 - 2\alpha\delta)^{k-\tau^*}\\
    &\ \leq \prod_{i=0}^k (1 - 2\alpha\delta(1+b_i))\norm{\theta_{0}-\theta^\star}_2^2 + \frac{15\alpha G_{\max}^2}{2\delta} + \frac{2\eta_2\tau^*\alpha}{\delta} +   \frac{8D_{\max}G_{\max}\kappa}{\delta}\\
    &\ \quad + \beta\left( \frac{2\eta_1\tau^*}{\delta} + \frac{5 D_{\max}^2 + 2\alpha D_{\max}G_{\max} + 4\alpha \eta_1\tau^*\lambda }{1-2\delta \alpha - \lambda} \right)(1-2\delta \alpha)^{k-\tau^*}.
\end{align*}

\section{ Proof of Theorem \ref{thm:acclinearNoniidDimin} }

Before proving this theorem, we introduce two lemmas of series sum that will help to streamline the presentation.

\begin{lemma}\label{lem:seqSum}
Let $a_k=\frac{\alpha}{\sqrt{k}}$ and $\beta_{k} = \beta \lambda^k$ with $\alpha>0,\beta,\lambda\in(0,1)$ for $k=1,2,\dots$. Then
\begin{equation}
    \sum_{k=1}^T \frac{\beta_{k}}{a_k} \leq \frac{\beta}{\alpha(1-\lambda)^2}.
\end{equation}
\end{lemma}
\begin{proof}
The proof is based on taking the standard sum of geometric sequences as follows:
\begin{equation}\nonumber
    \sum_{k=1}^T \frac{\beta_{k}}{a_k} = \sum_{k=1}^T \frac{\beta\lambda^{k}\sqrt{k}}{\alpha} \leq \sum_{k=1}^T \frac{\beta\lambda^{k}k} {\alpha} = \frac{\beta}{\alpha(1-\lambda)}\left(\sum_{k=1}^T \lambda^{k} -T\lambda^T\right) \leq \frac{\beta}{\alpha(1-\lambda)^2}.
\end{equation}
\end{proof}

\begin{lemma}\label{lem:seqSum2}
Let $a_k=\frac{\alpha}{\sqrt{k}}$. Then
\begin{equation}
    \sum_{k=1}^T a_k \leq 2\alpha \sqrt{T}.
\end{equation}
\end{lemma}
\begin{proof}
We use the comparison principle to bound the series sum as follows:
$$\sum_{k=1}^T a_k = \sum_{k=1}^T \frac{\alpha}{\sqrt{k}} \leq \int_1^{T+1}\frac{\alpha}{\sqrt{t-1}}dt = 2\alpha\sqrt{t-1}\rvert_1^{T+1} = 2\alpha \sqrt{T}.$$
\end{proof}

The proof of Theorem \ref{thm:acclinearNoniidDimin} is partially similar to that of Theorem \ref{thm:acclinearNoniid}. The steps are the same until \eqref{eq:pf_diffstart2}, where we have
\begin{align*}
    & \mathbb{E}\norm{\theta_{k+1}-\theta^\star}_2^2\\
    &\quad\leq \mathbb{E}\norm{\theta_{k}-\theta^\star}_2^2 + 5\beta_k D_{\max}^2 + 15a_k^2G_{\max}^2 + 2a_k b_k D_{\max}G_{\max}  - 2a_k(1+b_k)\delta\mathbb{E}\norm{\theta_{k}-\theta^\star}_2^2\\
    &\quad\quad - 2a_k(1+b_k)\mathbb{E} [\xi(\theta_k;O_k)].
\end{align*}
Then we continue the proof with rearranging the previous inequality:
\begin{align*}
    &2\delta\mathbb{E}\norm{\theta_{k}-\theta^\star}_2^2\\
    &\leq 2(1+b_k)\delta\mathbb{E}\norm{\theta_{k}-\theta^\star}_2^2\\
    &\leq \frac{\mathbb{E}\norm{\theta_{k}-\theta^\star}_2^2 - \mathbb{E}\norm{\theta_{k+1}-\theta^\star}_2^2}{a_k}\!  +\! \frac{5\beta_k}{a_k} D_{\max}^2\! +\!15a_k G_{\max}^2 \!+\! 2b_k D_{\max}G_{\max} \!+\! 4 |\mathbb{E} [\xi(\theta_k;O_k)]| .
\end{align*}
Then we sum over time step $k$ from 1 to $T(T>\tau^*)$ and obtain
\begin{align*}
    2\delta\sum_{k=1}^T&\mathbb{E}\norm{\theta_{k}-\theta^\star}_2^2\\
    \leq& \sum_{k=1}^T\frac{\mathbb{E}\norm{\theta_{k}-\theta^\star}_2^2- \mathbb{E}\norm{\theta_{k+1}-\theta^\star}_2^2}{a_k} + 4\sum_{k=1}^{\tau^*}|\mathbb{E} [\xi(\theta_k;O_k)]| + 4\sum_{k=\tau^*+1}^{T}|\mathbb{E} [\xi(\theta_k;O_k)]|\\
    &+ 5D_{\max}^2\sum_{k=1}^T\frac{\beta_k}{a_k}  + 15G_{\max}^2\sum_{k=1}^Ta_k  + 2D_{\max}G_{\max}\sum_{k=1}^Tb_k\\
    =& \frac{\norm{\theta_{1}-\theta^\star}_2^2}{a_1} + \sum_{k=2}^T\mathbb{E}\norm{\theta_{k}-\theta^\star}_2^2\left(\frac{1}{a_k}-\frac{1}{a_{k-1}}\right)-\frac{\mathbb{E}\norm{\theta_{T+1}-\theta^\star}_2^2}{a_{T+1}} \\
    &+ 5D_{\max}^2\sum_{k=1}^T\frac{\beta_k}{a_k}  + 15G_{\max}^2\sum_{k=1}^Ta_k  + 2D_{\max}G_{\max}\sum_{k=1}^Tb_k\\
    &+ 4\sum_{k=1}^{\tau^*}|\mathbb{E} [\xi(\theta_k;O_k)]| + 4\sum_{k=\tau^*+1}^{T}|\mathbb{E} [\xi(\theta_k;O_k)]|\\
    \overset{\text{(i)}}{\leq}& \frac{\norm{\theta_{1}-\theta^\star}_2^2}{a_1} + D_{\max}^2 \sum_{k=2}^T\left(\frac{1}{a_k}-\frac{1}{a_{k-1}}\right) \\
    &+ 5D_{\max}^2\sum_{k=1}^T\frac{\beta_k}{a_k}  + 15G_{\max}^2\sum_{k=1}^Ta_k  + 2D_{\max}G_{\max}\sum_{k=1}^Tb_k\\
    &+ 4\sum_{k=1}^{\tau^*}|\mathbb{E} [\xi(\theta_k;O_k)]| + 4\sum_{k=\tau^*+1}^{T}|\mathbb{E} [\xi(\theta_k;O_k)]|\\
    \overset{\text{(ii)}}{\leq}& \frac{D_{\max}^2}{\alpha_T} + 5D_{\max}^2\sum_{k=1}^T\frac{\beta_k}{a_k}  + 15G_{\max}^2\sum_{k=1}^Ta_k  + 2D_{\max}G_{\max}\sum_{k=1}^T\beta_k\\
    &+ 4\sum_{k=1}^{\tau^*}|\mathbb{E} [\xi(\theta_k;O_k)]| + 4\sum_{k=\tau^*+1}^{T}|\mathbb{E} [\xi(\theta_k;O_k)]|\\
    \overset{\text{(iii)}}{\leq}& \frac{D_{\max}^2\sqrt{T}}{\alpha} + \frac{5\beta D_{\max}^2}{\alpha(1-\lambda)^2} + 30\alpha G_{\max}^2 \sqrt{T} + \frac{2D_{\max}G_{\max}\beta \lambda}{1-\lambda} \\
    &+ 4\sum_{k=1}^{\tau^*}|\mathbb{E} [\xi(\theta_k;O_k)]| + 4\sum_{k=\tau^*+1}^{T}|\mathbb{E} [\xi(\theta_k;O_k)]|,
\end{align*}
where (i) follows from Assumption \ref{asp:boundedDomain} and the fact that $\alpha_{k} < \alpha_{k-1}$, and $\mathbb{E}\norm{\theta_{T+1}-\theta^\star}^2/a_{T+1}>0$, (ii) holds due to Assumption \ref{asp:boundedDomain}, and (iii) follows from Lemmas \ref{lem:boundedGra}, \ref{lem:seqSum}, and \ref{lem:seqSum2}.

It remains to bound $4\sum_{k=1}^{\tau^*}|\mathbb{E} [\xi(\theta_k;O_k)]| + 4\sum_{k=\tau^*+1}^{T}|\mathbb{E} [\xi(\theta_k;O_k)]|$. We bound the tail term by using Lemma \ref{lem:bias}.

For simplicity, in the following we denote 
$$\eta_1 = 2D_{\max}((1+\gamma)D_{\max}+G_{\max}),\quad \eta_2 = 6G_{\max}((1+\gamma)D_{\max}+G_{\max}).$$
Following from Lemma \ref{lem:bias}, we have
\begin{align*}
\sum_{k=1}^{\tau^*}|\mathbb{E} [\xi(\theta_k;O_k)]| \leq& \sum_{k=1}^{\tau^*}\eta_1\sum_{i=1}^{k-1}\beta_i + \sum_{k=1}^{\tau^*}\eta_2\sum_{i=1}^{k-1}a_i \\
\leq& \tau^*\eta_1\sum_{k=1}^{T}\beta_k + \tau^*\eta_2\sum_{k=1}^{T}a_k\\
\leq& \frac{\tau^*\eta_1\beta\lambda}{1-\lambda} + 2\tau^*\eta_2\alpha\sqrt{T}.
\end{align*}
Similarly, we obtain
\begin{align*}
\sum_{k=\tau^*+1}^{T}|\mathbb{E} [\xi(\theta_k;O_k)]| \leq& \sum_{k=\tau^*+1}^{T}\left( 4 D_{\max}G_{\max}\kappa + \eta_1\tau^*\beta_{k-\tau^*} + \eta_2\tau^*a_{k-\tau^*}\right) \\
\leq& 4 D_{\max}G_{\max}\kappa T + \tau^*\eta_1\sum_{k=1}^{T-\tau^*}\beta_k + \tau^*\eta_2\sum_{k=1}^{T-\tau^*}a_k\\
\leq& 4 D_{\max}G_{\max}\kappa T + \frac{\tau^*\eta_1\beta\lambda}{1-\lambda} + 2\tau^*\eta_2\alpha\sqrt{T}.
\end{align*}
Thus, we have
\begin{align*}
    2\delta\sum_{k=1}^T&\mathbb{E}\norm{\theta_{k}-\theta^\star}_2^2\\
    \leq& \frac{D_{\max}^2\sqrt{T}}{\alpha} + \frac{5\beta D_{\max}^2}{\alpha(1-\lambda)^2} + 30\alpha G_{\max}^2 \sqrt{T} + \frac{2D_{\max}G_{\max}\beta \lambda}{1-\lambda}\\
    &+ 4\sum_{k=1}^{\tau^*}|\mathbb{E} [\xi(\theta_k;O_k)]| + 4\sum_{k=\tau^*+1}^{T}|\mathbb{E} [\xi(\theta_k;O_k)]|\\
     \leq& \frac{D_{\max}^2\sqrt{T}}{\alpha} + \frac{5\beta D_{\max}^2}{\alpha(1-\lambda)^2} + 30\alpha G_{\max}^2 \sqrt{T} + \frac{2D_{\max}G_{\max}\beta \lambda}{1-\lambda}\\
     &+ 16 D_{\max}G_{\max}\kappa T + \frac{8\tau^*\eta_1\beta\lambda}{1-\lambda} + 16\tau^*\eta_2\alpha\sqrt{T}.
\end{align*}

Finally, we apply Jensen's inequality and complete the proof as
\begin{align*}
    \mathbb{E}\norm{\theta_{\text{out}}-\theta^\star}_2^2 =& \mathbb{E}\norm{\frac{1}{T}\sum_{k=1}^T\theta_{k}-\theta^\star}_2^2
    \leq \frac{1}{T}\sum_{k=1}^T\mathbb{E}\norm{\theta_{k}-\theta^\star}_2^2\\
    \leq& \frac{ D_{\max}^2/\alpha + 30\alpha G_{\max}^2 + 16\tau^*\alpha\eta_2}{2\delta\sqrt{T}} + \frac{8 D_{\max}G_{\max}\kappa}{\delta}\\
    &+ \frac{1}{T}\left[\frac{5\beta D_{\max}^2}{2\alpha\delta(1-\lambda)^2} + \frac{D_{\max}G_{\max}\beta \lambda + 4\tau^*\eta_1\beta\lambda}{\delta(1-\lambda)} \right].
\end{align*}

\section{Proof of Proposition~\ref{lem:boundDk}}

\begin{proof}
For convenience, we denote $\mathcal{M}Q_k(y_k):=\max_{u\in U(y_k)}Q_k(y_k,u)$, then $\hat{\mathcal{T}}_kQ_k = R + \mathcal{M}Q_k(y_k)$ and $\hat{\mathcal{T}}_kQ_{k-1} = R + \mathcal{M}Q_{k-1}(y_k)$.
If $k=0$, we have from \eqref{eq:Dk} that
\begin{align*}
\left\Vert \mathcal{D}_{0}\left[Q_{0},Q_{-1}\right]\right\Vert = & \left\Vert \hat{\mathcal{T}}_{0}Q_{0}\right\Vert \leq\left\Vert R\right\Vert +\gamma\left\Vert \mathcal{M}Q_{0}(y_{0})\right\Vert \\
\leq & R_{\max}+\gamma V_{\max}.
\end{align*}

Now, considering $k\geq1$ we have
\begin{align}
\label{eq:DkWriteOut}
&\left\Vert \mathcal{D}_{k}\left[Q_{k},Q_{k-1}\right]\right\Vert \nonumber\\
&\quad \overset{\text{(i)}}{\leq} \left\Vert R\right\Vert +\gamma\Vert(1+b_k)\mathcal{M}Q_{k}-b_k\mathcal{M}Q_{k-1}\Vert \nonumber\\
&\quad \leq R_{\max}+\gamma\Vert(1+b_k)\mathcal{M}\bigl(Q_{k-1} -\alpha_{k-1}Q_{k-2}+\alpha_{k-1}\mathcal{D}_{k-1}\left[Q_{k-1},Q_{k-2}\right]\bigr) -b_k\mathcal{M}Q_{k-1}\Vert \nonumber \\
&\quad\overset{\text{(ii)}}{\leq}  R_{\max}+\gamma\norm{Q_{k-1}} + \gamma |1+b_k| a_{k-1}\norm{Q_{k-2}}+\gamma |1+b_k|\alpha_{k-1}\norm{\mathcal{D}_{k-1}\left[Q_{k-1},Q_{k-2}\right]},
\end{align}
where (i) follows from the triangle inequality and (ii) follows from the definition of the infinity norm. 

To proceed to bound \eqref{eq:DkWriteOut}, we consider two cases. If $k<\frac{m}{2}$, there are at most a finite number of $\mathcal{D}_{k}$'s, which are obviously bounded. If $k\geq \frac{m}{2}$, we have $|1+b_k|a_{k-1} =\frac{|k-m|}{k}\leq 1$. It follows from \eqref{eq:DkWriteOut} that
\begin{align}
\label{eq:DkBound}
&\left\Vert \mathcal{D}_{k}\left[Q_{k},Q_{k-1}\right]\right\Vert \nonumber\\
&\quad \leq  R_{\max}+\gamma\norm{Q_{k-1}} + \gamma \norm{Q_{k-2}}+ \gamma\norm{\mathcal{D}_{k-1}\left[Q_{k-1},Q_{k-2}\right]} \nonumber \\
&\quad \overset{\text{(i)}}{\leq} R_{\max}+2\gamma V_{\max}+\gamma\left\Vert \mathcal{D}_{k-1}\left[Q_{k-1},Q_{k-2}\right]\right\Vert \nonumber\\
&\quad \overset{\text{(ii)}}{\leq} R_{\max}\sum_{i=0}^{k-\lfloor m/2 \rfloor}\gamma^{i}+2V_{\max}\sum_{i=1}^{k-\lfloor m/2 \rfloor}\gamma^{i} + \gamma^{k-\lfloor m/2 \rfloor}\left\Vert \mathcal{D}_{\lfloor m/2 \rfloor}\left[Q_{\lfloor m/2 \rfloor},Q_{\lfloor m/2 \rfloor-1}\right]\right\Vert
\end{align}
where $\lfloor x \rfloor$ denotes the largest integer that is no larger than $x$. Note that (i) follows from the boundedness of $Q_k$ (Assumption~\ref{asp:boundQ}), and (ii) follows from applying (i) to $\mathcal{D}_t$ for $t=k-1, k-2, \dots, {\lfloor m/2 \rfloor}+1$ iteratively. Since $\gamma <1$, the first two items in (ii) are bounded. Obviously, the third item is also bounded. Therefore, there exists some constant $\bar{D}$, such that $\left\Vert \mathcal{D}_{k}\right\Vert\leq \bar{D},\ \forall k\geq 0$.

The bound on $\epsilon_{k}$ follows directly from its definition as
$$\begin{aligned}
\left\Vert \epsilon_{k}\right\Vert &=\left\Vert \mathbb{E}_{P}\left(\mathcal{D}_{k}\left[Q_{k},Q_{k-1}\right](x,u)|\mathcal{F}_{k-1}\right)-\mathcal{D}_{k}\left[Q_{k},Q_{k-1}\right]\right\Vert\\
&\leq2\left\Vert \mathcal{D}_{k}\left[Q_{k},Q_{k-1}\right]\right\Vert\leq 2\bar{D}.
\end{aligned}$$
Thus we conclude our proof.
\end{proof}

\section{Proof of Theorem~\ref{thm:main} }

We first prove two lemmas that will be useful for establishing the main results.
The first lemma derives the dynamics of $Q_k$ in terms of $E_k$..
\begin{lemma}\label{lem:dynQ}
Consider MomentumQ  as in Algorithm~\ref{alg:AQL}. For any $k\geq1$, we have 
\begin{align}
Q_{k}=\frac{1}{k}(Q_{k-1}-Q_{0}+(k-m-1)\mathcal{T}Q_{k-1})+\frac{1}{k}((m+1)\mathcal{T}Q_{0}-E_{k-1}). \label{eq:2termIter}
\end{align}
\end{lemma}
\begin{proof}
We prove the lemma by substituting the learning rates $a_k,b_k,c_k$ in Algorithm~\ref{alg:AQL} and using induction. From \eqref{eq:compactAQL}, we see that $Q_{1}=\hat{\mathcal{T}}_{1}Q_{0}=\mathcal{T}Q_{0}-E_{0}$, Thus (\ref{eq:2termIter}) holds when $k=1$. Now assume (\ref{eq:2termIter})
holds for a certain integer $k>1$ we prove it also holds for $k+1$. To see this, we rewrite \eqref{eq:compactAQL} as
\begin{align*}
Q_{k+1}
= & \frac{1}{k+1}Q_{k}-\frac{1}{k+1}Q_{k-1}+\frac{k}{k+1}Q_{k} +\frac{1}{k+1}\left[(k-m)\hat{\mathcal{T}}_{k}Q_{k}-(k-m-1)\hat{\mathcal{T}}_{k}Q_{k-1}\right]\\
= & \frac{1}{k+1}Q_{k}-\frac{1}{k+1}Q_{k-1}+\frac{1}{k+1}(Q_{k-1}-Q_{0} +(k-m-1)\mathcal{T}Q_{k-1}\\
 &+(m+1)\mathcal{T}Q_{0}-E_{k-1}) +\frac{1}{k+1}\left[(k-m)\hat{\mathcal{T}}_{k}Q_{k}-(k-m-1)\hat{\mathcal{T}}_{k}Q_{k-1}\right]\\
= & \frac{1}{k+1}Q_{k}-\frac{1}{k+1}Q_{k-1}+\frac{1}{k+1}(Q_{k-1}-Q_{0} +(k-m-1)\mathcal{T}Q_{k-1}\\
 & +(m+1)\mathcal{T}Q_{0}-E_{k-1})+\frac{1}{k+1}\left[(k-m)\mathcal{T}Q_{k}-(k-m-1)\mathcal{T}Q_{k-1}-\epsilon_{k}\right]\\
= & \frac{1}{k+1}(Q_{k}-Q_{0}+(k-m)\mathcal{T}Q_{k}+(m+1)\mathcal{T}Q_{0}-E_{k}),
\end{align*}
which shows that (\ref{eq:2termIter}) holds for $k+1$. Therefore, by induction (\ref{eq:2termIter}) holds for all $k\geq1$.
\end{proof}

The second lemma derives the propagation of the approximation errors $\epsilon_k$ in the process of Q-function iteration, which can be proved conveniently using Lemma~\ref{lem:dynQ}.
\begin{lemma}\label{lem:errorProp}
Suppose Assumption~\ref{asp:boundQ} holds and $m\geq \frac{1}{\gamma}$ as in Algorithm~\ref{alg:AQL}. Then for all $k\geq m+1$, we have
\begin{equation}\label{eq:pf5}
    \left\Vert Q^{\star}\!-\!Q_{k}\right\Vert\leq \frac{\tilde{h}V_{\max}}{k(1-\gamma)}\!+\!\frac{1}{k}\sum_{i=0}^{k-\lfloor m \rfloor-2}\gamma^{i}\Vert E_{k-i}\Vert,
\end{equation}
where $\tilde{h}=2\gamma(m+\lfloor m \rfloor+2)+2$.
\end{lemma}
\begin{proof}
For $k\geq m+1$, expand $Q_k$ using (\ref{eq:2termIter}) in Lemma~\ref{lem:dynQ} iteratively, yielding
\begin{align*}
\left\Vert Q^{\star}-Q_{k}\right\Vert 
= & \frac{1}{k}\Vert Q_{0}-Q_{k-1}+(k-m-1)(\mathcal{T}Q^{\star}-\mathcal{T}Q_{k-1}) +(m+1)(\mathcal{T}Q^{\star}-\mathcal{T}Q_{0})+E_{k}\Vert\\
\overset{\text{(i)}}{\leq} & \frac{\gamma(k-m-1)+1}{k}\Vert Q^{\star}-Q_{k-1}\Vert +\frac{\gamma(m+1)+1}{k}\Vert Q^{\star}-Q_{0}\Vert+\frac{\Vert E_{k}\Vert}{k}\\
\overset{\text{(ii)}}{\leq} & \frac{\gamma(k-1)}{k}\Vert Q^{\star}-Q_{k-1}\Vert+\frac{2h}{k}V_{\max}+\frac{\Vert E_{k}\Vert}{k}\\
\overset{\text{(iii)}}{\leq} & \frac{\gamma^{k-\lfloor m \rfloor-1}}{k}(\lfloor m \rfloor+1)\Vert Q^{\star}-Q_{\lfloor m \rfloor+1}\Vert+\frac{2hV_{\max}}{k}\sum_{i=0}^{k-\lfloor m \rfloor-2}\gamma^{i} +\sum_{i=0}^{k-\lfloor m \rfloor-2}\frac{\gamma^{i}}{k}\Vert E_{k-i}\Vert\\
\leq & 2\frac{\gamma (\lfloor m \rfloor+1)+h}{k(1-\gamma)}V_{\max}+\frac{1}{k}\sum_{i=0}^{k-\lfloor m \rfloor-2}\gamma^{i}\Vert E_{k-i}\Vert,
\end{align*}
where (i) follows from the triangle inequality and the contraction property~\eqref{eq:Contraction}, (ii) follows from Assumption~\ref{asp:boundQ} and because $m\geq\frac{1}{\gamma}$, $h=\gamma(m+1)+1$, and (iii) follows from applying (ii) to $\norm{Q^{\star}-Q_t}$ for $t=k-1, k-2, \dots, \lfloor m \rfloor+2$ iteratively. Then \eqref{eq:pf5} follows from the definition of $\tilde{h}$.
\end{proof}

\begin{lemma}\label{lem:MHAineq}
(Maximal Hoeffding-Azuma Inequality) \citep[Chapter 7]{Alon2008}\\ Let $\{M_1,M_2,\dots,M_T\}$ be a martingale difference sequence with respect to a sequence of random variables $\{X_1,X_2,\dots, X_T\}$ (i.e. $\mathbb{E}(M_{k+1}|X_1,X_2,\dots,X_k)=0,\forall 1\leq k\leq T$) and uniformly bounded by $\bar M>0$ almost surely. If we define $S_k=\sum_{i=1}^k M_i$, then for any $\varepsilon>0$, we have
$$ \mathbb{P}\left( \underset{1\leq k\leq T}{\max}S_k>\varepsilon \right)\leq \exp\left( \frac{-\varepsilon^2}{2T\bar M^2} \right).
$$
\end{lemma}
Now we are ready to prove the main results of Theorem~\ref{thm:main}. 

\begin{proof}[Proof of Theorem~\ref{thm:main}] 
The proof applies Lemma~\ref{lem:errorProp} and the Maximal Hoeffding-Azuma Inequality (Lemma~\ref{lem:MHAineq}).

 Applying Lemma~\ref{lem:errorProp} with $k=T$, we obtain
$$
    \left\Vert Q^{\star}-Q_{T}\right\Vert\leq \frac{\tilde{h}V_{\max}}{T(1-\gamma)}+\frac{1}{T}\sum_{i=0}^{T-\lfloor m \rfloor-2}\gamma^{i}\Vert E_{T-i}\Vert.
$$
It suffices to bound the second term. For convenience, we denote $K=T-\lfloor m \rfloor-2$. Observe that 
\begin{equation}
    \frac{1}{T}\sum_{i=0}^{K}\gamma^{i}\Vert E_{T-i}\Vert
    \leq\frac{1}{T}\sum_{i=0}^{K}\gamma^{i}\underset{0\leq i\leq K}{\max}\norm{E_{T-i}}\leq \frac{\max_{0\leq i\leq K}\norm{E_{T-i}}}{(1-\gamma)T}. \label{eq:pf3}
\end{equation}

In remains to bound $\underset{0\leq i\leq K}{\max}\norm{E_{T-i}}$. Notice that $\underset{0\leq i\leq K}{\max}\norm{E_{T-i}}=\underset{(x,u)}{\max} \underset{0\leq i\leq K}{\max}\lvert{E_{T-i}(x,u)}\rvert$. For a given 
$(x,u)$ and $\varepsilon>0$, we have
\begin{align}
    &\mathbb{P}\left( \underset{0\leq i\leq K}{\max}\lvert E_{T-i}(x,u)\rvert>\varepsilon \right)\nonumber\\
    &\quad=\mathbb{P}\left( \left\{\underset{0\leq i\leq K}{\max}( E_{T-i}(x,u)) >\varepsilon\right\} \bigcup \left\{\underset{0\leq i\leq K}{\max}( -E_{T-i}(x,u)) >\varepsilon\right\} \right)\nonumber\\
    &\quad=\mathbb{P}\left( \underset{0\leq i\leq K}{\max}( E_{T-i}(x,u))>\varepsilon \right) + \mathbb{P}\left( \underset{0\leq i\leq K}{\max}( -E_{T-i}(x,u))>\varepsilon \right), \label{eq:pf1}
\end{align}
where $\bar{D}$ is specified in Proposition~\ref{lem:boundDk}.
Since $\{\epsilon_k(x,u)\}_{k\geq 0}$ is a martingale difference sequence with respect to the filtration $\mathcal{F}_k$ as defined previously, we apply the Maximal Hoeffding-Azuma inequality (see Lemma~\ref{lem:MHAineq}) and obtain
\begin{equation*}
\mathbb{P}\left( \underset{0\leq i\leq K}{\max}( E_{T-i}(x,u))>\varepsilon\right) \leq \exp\left( \frac{-\varepsilon^2}{8(K+1)\bar{D}^2} \right),
\end{equation*}
and
\begin{equation*}
    \mathbb{P}\left( \underset{0\leq i\leq K}{\max}( -E_{T-i}(x,u))>\varepsilon\right) \leq \exp\left( \frac{-\varepsilon^2}{8(K+1)\bar{D}^2} \right).
\end{equation*}
Then we further bound~\eqref{eq:pf1} as
$$ \mathbb{P}\left( \underset{0\leq i\leq K}{\max}\lvert E_{T-i}(x,u)\rvert>\varepsilon \right)\leq 2\exp\left( \frac{-\varepsilon^2}{8(K+1)\bar{D}^2} \right).
$$
Since we consider a finite state-action space where the number of state-action pairs is defined by $n$, we use the union bound to obtain
$$ \mathbb{P}\left( \underset{0\leq i\leq K}{\max}\lVert E_{T-i}\rVert>\varepsilon \right)\leq 2n\exp\left( \frac{-\varepsilon^2}{8(K+1)\bar{D}^2} \right).
$$
Letting $\delta = 2n\exp\left( \frac{-\varepsilon^2}{8(K+1)\bar{D}^2} \right)$, and we have
$$ \mathbb{P}\left( \underset{0\leq i\leq K}{\max}\lVert E_{T-i}\rVert\leq \bar{D}\sqrt{8(K+1)\log \frac{2n}{\delta}} \right)\geq 1-\delta,
$$
where $K=T-\lfloor m \rfloor-2$.
By substituting the above bound into~\eqref{eq:pf3} yields the desired result.
\end{proof}

\vskip 0.2in
\bibliography{jmlr}

\end{document}